%% file: main.tex
\documentclass[10pt,twocolumn,letterpaper]{article}

\usepackage[final]{cvpr}

\input{preamble}

\definecolor{cvprblue}{rgb}{0.21,0.49,0.74}
\usepackage[pagebackref,breaklinks,colorlinks,allcolors=cvprblue]{hyperref}
\usepackage[utf8]{inputenc}
\usepackage{multirow}

\usepackage{subcaption}

\def\confName{CVPR}
\def\confYear{2026}
\usepackage{xcolor}
\usepackage{soul}

\title{VLA Models Are More Generalizable Than You Think:\\ Revisiting Physical and Spatial Modeling}

\author{
  Weiqi Li$^{1}$ \quad
  Quande Zhang$^{1}$ \quad
  Ruifeng Zhai$^{1}$ \quad
  Liang Lin$^{1,2,3}$ \quad
  Guangrun Wang$^{1,2,3,*}$\\
  $^1$Sun Yat-sen University \quad $^2$Guangdong Key Laboratory of Big Data Analysis and Processing\\ $^3$X-Era AI Lab\\
  {\tt\small liwq229@mail2.sysu.edu.cn, zhangqd8@mail2.sysu.edu.cn, wanggrun@gmail.com}
}

\begin{document}

\twocolumn[{
\renewcommand\twocolumn[1][]{#1}
\maketitle

\begin{center}
\centering
\captionsetup{type=figure}

\hspace{-0.5cm}
\begin{subfigure}[b]{0.48\textwidth}
    \centering
    \includegraphics[height=4.7cm, keepaspectratio]{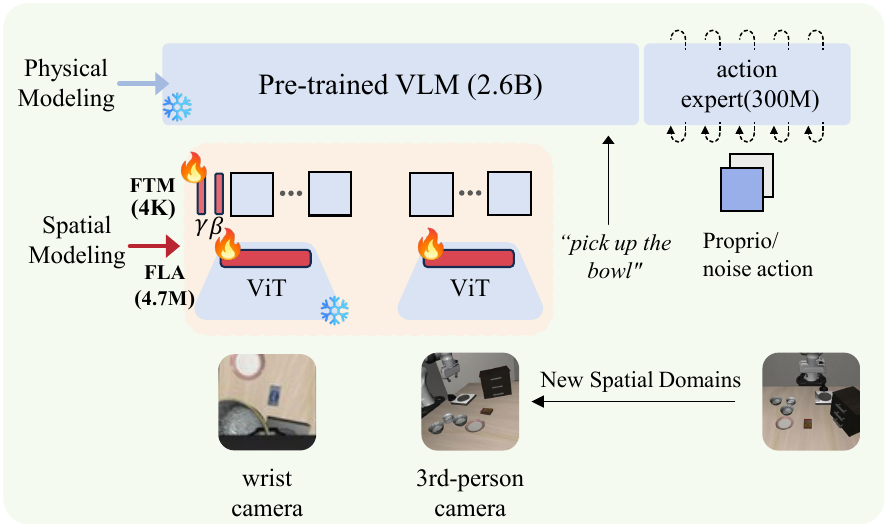}
    \caption{}
\end{subfigure}
\hspace{-0.02\textwidth}
\begin{subfigure}[b]{0.48\textwidth}
    \centering
    \includegraphics[height=4.7cm, keepaspectratio]{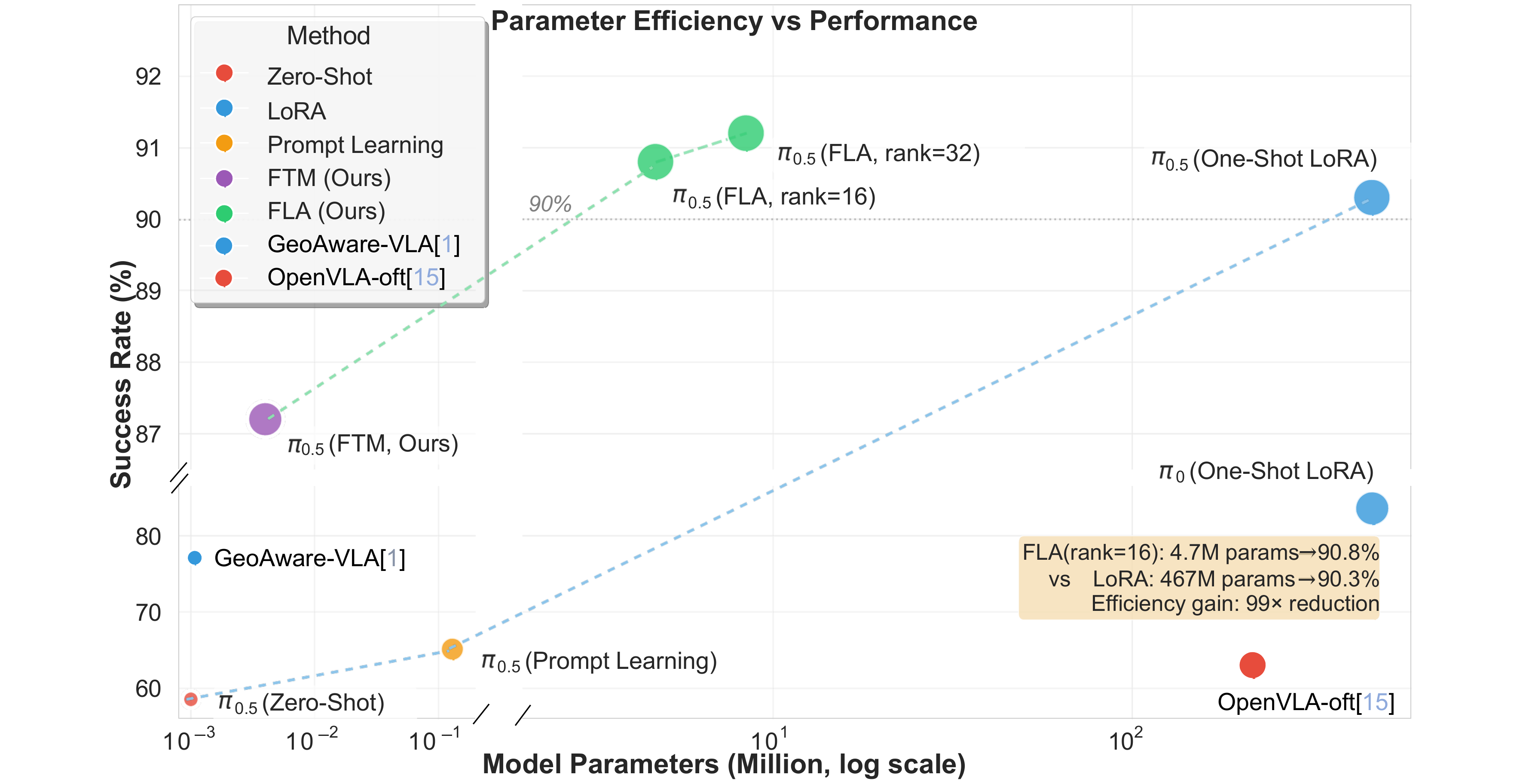}
    \caption{}
\end{subfigure}

\caption{\textbf{Illustration of Spatial Modeling Adaptation.} 
(a) Our two proposed one-shot adaptation methods—Feature Token Modulation (FTM) and Feature Linear Adaptation (FLA)—which adapt visual representations to new spatial domains. After adaptation, all multimodal tokens are processed by the pretrained VLM and action expert to generate the final policy. 
(b) Parameter efficiency versus performance on the LIBERO benchmark under novel viewpoints. Our model exceeds the average success rate of the $\pi_{0.5}$ (LoRA) baseline while using only 4.7M parameters compared to 467M.}
\label{fig:teaser}

\end{center}
}]

\begingroup
  \renewcommand\thefootnote{\fnsymbol{footnote}} 
  \footnotetext[1]{Corresponding author: Guangrun Wang.}
\endgroup

\maketitle
\input{sec/0_abstract}    
\input{sec/1_intro}
\input{sec/2_related}
\input{sec/3_method}

\input{sec/4_experiment}

{
    \small
    \bibliographystyle{ieeenat_fullname}
    \bibliography{main}
}

\input{sec/X_suppl}

\end{document}

%% file: preamble.tex
\usepackage{adjustbox}
\usepackage{amsthm}

%% file: sec/0_abstract.tex
\begin{abstract}
Vision-language-action (VLA) models achieve strong in-distribution performance but degrade sharply under novel camera viewpoints and visual perturbations. We show that this brittleness primarily arises from misalignment in Spatial Modeling, rather than Physical Modeling.  
To address this, we propose a one-shot adaptation framework that recalibrates visual representations through lightweight, learnable updates. Our first method, Feature Token Modulation (FTM), applies a global affine transformation to visual tokens and improves Libero viewpoint accuracy from 48.5\% to 87.1\% with only 4K parameters. Building on this, Feature Linear Adaptation (FLA) introduces low-rank updates to the ViT encoder, achieving 90.8\% success with 4.7M parameters—matching LoRA-scale finetuning at far lower cost.  
Together, these results reveal substantial untapped robustness in pretrained VLA models and demonstrate that targeted, minimal visual adaptation is sufficient to restore viewpoint generalization.
\end{abstract}

%% file: sec/1_intro.tex
\section{Introduction}
\label{sec:intro}

Recent advances in vision-language-action (VLA) models have significantly advanced the field of embodied intelligence. By extending pretrained vision-language foundations~\cite{li2024llavaonevisioneasyvisualtask,embodimentcollaboration2025openxembodimentroboticlearning,durante2024interactiveagentfoundationmodel,huang2024embodiedgeneralistagent3d,kim2024openvlaopensourcevisionlanguageactionmodel,li2024visionlanguagefoundationmodelseffective,zhen20243dvla3dvisionlanguageactiongenerative,octomodelteam2024octoopensourcegeneralistrobot} to robotic control, these models can interpret visual observations and follow natural language instructions across diverse manipulation tasks~\cite{brohan2023rt2visionlanguageactionmodelstransfer,black2024pi0visionlanguageactionflowmodel}. Pretraining on large-scale and varied robot datasets~\cite{khazatsky2025droidlargescaleinthewildrobot,embodimentcollaboration2025openxembodimentroboticlearning,walke2024bridgedatav2datasetrobot} has led to strong in-domain performance, demonstrating the potential of language-conditioned visual grounding for general-purpose robots.

However, robustness and generalization remain critical challenges~\cite{xing2021kitchenshift,yu2021metaworldbenchmarkevaluationmultitask,li2024evaluatingrealworldrobotmanipulation,xie2023decomposinggeneralizationgapimitation,zhang2024vlabenchlargescalebenchmarklanguageconditioned,liu2023liberobenchmarkingknowledgetransfer,zhan2026stablelanguageguidancevisionlanguageaction}. Despite large and diverse training corpora, VLA models often fail to generalize to out-of-distribution scenarios, exhibiting pronounced degradation when exposed to unseen camera viewpoints, lighting changes, or scene perturbations~\cite{zhou2025liberoprorobustfairevaluation,fei2025liberoplusindepthrobustnessanalysis}. Such brittleness limits their applicability in real-world environments, where visual conditions are inherently dynamic and unpredictable \cite{chen2026radar}.

Existing approaches to improving VLA robustness can be broadly categorized into data-centric and representation-centric paradigms. 
Data-centric approaches seek to enhance robustness by increasing visual diversity, typically through large-scale, multi-view robotic datasets such as Libero-Plus~\cite{fei2025liberoplusindepthrobustnessanalysis}. 
Although this strategy exposes models to a wider range of visual conditions, scaling data collection to real-world settings remains costly and labor-intensive, making continual adaptation impractical.
In contrast, representation-centric approaches aim to achieve robustness through geometric consistency, learning geometry-aware or 3D-consistent representations~\cite{wilcox2025adapt3radaptive3dscene,yang2025fp33dfoundationpolicy,ze20243ddiffusionpolicygeneralizable,abouzeid2025geoawarevlaimplicitgeometryaware,pang2025reviwo}. 
By leveraging multi-view supervision~\cite{wilcox2025adapt3radaptive3dscene} or specialized 3D-aware architectures~\cite{abouzeid2025geoawarevlaimplicitgeometryaware}, these methods improve viewpoint invariance. 
However, despite being geometry-aware, these methods remain sensitive to task-irrelevant visual factors such as background clutter and illumination changes~\cite{wilcox2025adapt3radaptive3dscene}, highlighting a persistent gap between viewpoint robustness and general visual adaptability. 
Prior methods implicitly assume that robustness requires additional data or 3D-consistent architectures, yet few works explicitly identify whether the failure actually arises from the spatial representation itself.

To investigate this question, motivated by recent discussions on the limitations of visual spatial reasoning in VLAs~\cite{abouzeid2025geoawarevlaimplicitgeometryaware, zhang2026vlm4vlarevisitingvisionlanguagemodelsvisionlanguageaction, qu2025spatialvlaexploringspatialrepresentations}, we propose a conceptual framework that decomposes a VLA model into two components~(see Fig.\ref{fig:teaser}(a)):
(1) \textbf{Spatial Modeling}, implemented by the visual encoder, which constructs spatial relationships between objects—such as positions, orientations, contact relations, and occlusion structure—from the image, producing the spatial representations required for downstream tasks;
(2) \textbf{Physical Modeling}, implemented by the VLM and the action expert, which integrates the task language, spatial representations, and action history to perform high-level reasoning and generate executable action sequences.
A viewpoint shift primarily alters the spatial configuration of the observed scene, rather than the underlying task semantics or action dynamics.
This suggests that the degradation of VLA performance under viewpoint changes is largely due to misalignment in Spatial Modeling, while the Physical Modeling components remain functionally competent.
In other words, the high-level policy still retains its reasoning and control capability, but receives spatially distorted visual embeddings, leading to failures in downstream coordination.

In light of the fact that Spatial Modeling is the primary bottleneck under viewpoint shifts, we propose a one-shot robustness adaptation framework that enables pretrained vision-language-action models to recalibrate their visual representations through lightweight, learnable modulation. Specifically, we introduce Feature Token Modulation (FTM), a simple yet effective mechanism that applies a global affine transformation to visual token embeddings using only two learnable parameter vectors $(\gamma, \beta)$. This operation re-scales and re-centers feature distributions, allowing the model to adapt rapidly to new visual domains. Despite involving fewer than 4K parameters, FTM significantly improves viewpoint robustness, raising accuracy from 48.5\% to 87.1\% (see Fig. \ref{fig:teaser}(b)).

Motivated by this result, we further propose Feature Linear Adaptation (FLA), a lightweight enhancement applied to the ViT encoder via low-rank updates to its linear layers. This adaptation realigns internal features with modest parameter overhead (4.7M parameters in total), further boosting the average success rate to 90.8\% (see Fig. \ref{fig:teaser}(b)), comparable to full-parameter fine-tuning approaches such as LoRA, but at a fraction of the cost.

\begin{figure}[t]
    \centering
    \includegraphics[width=\linewidth]{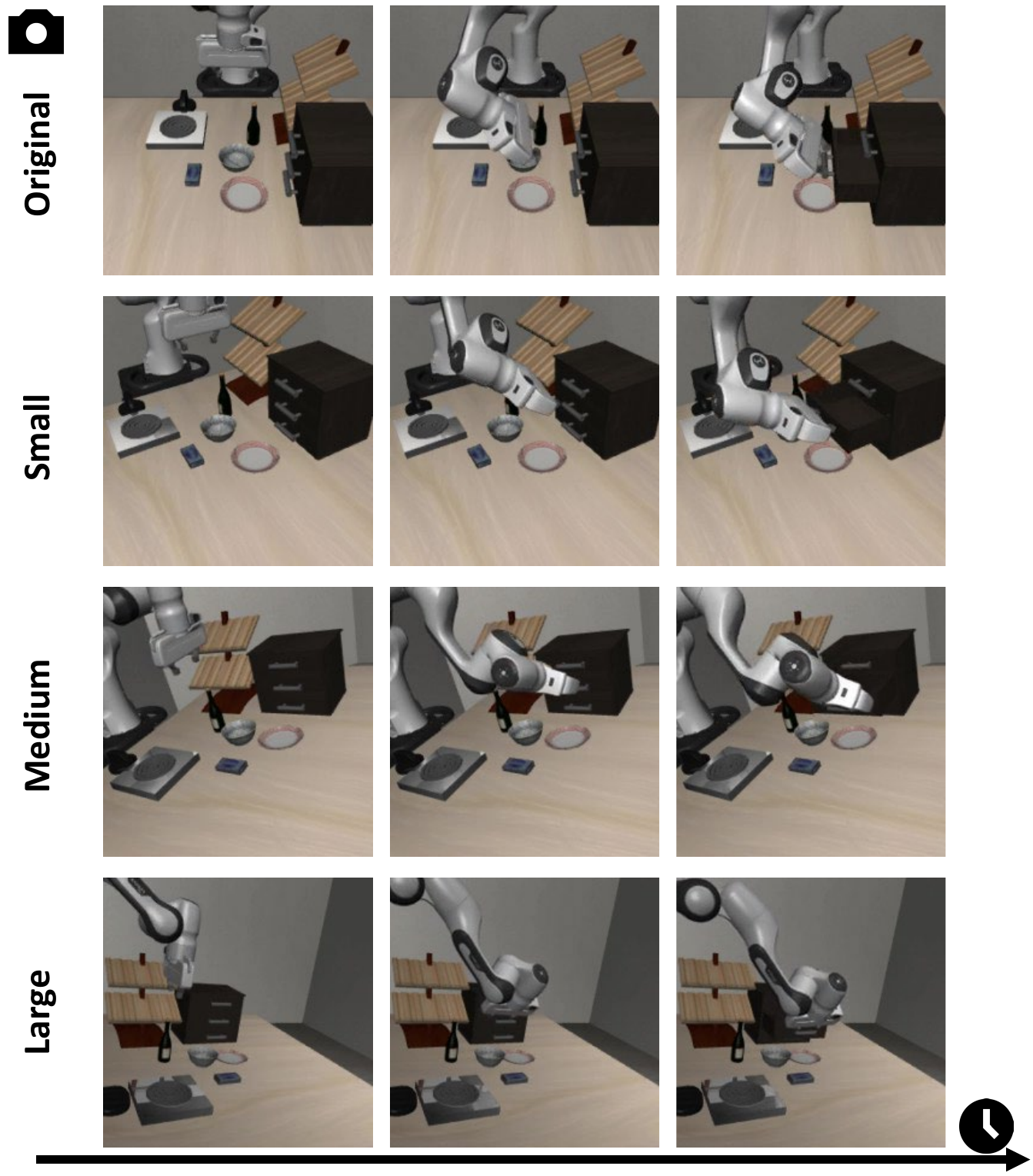}
    \caption{\textbf{Visualization of a sample rollout generated by our method.}
Each row shows how observations from different viewpoints evolve over time (columns). The rollout demonstrates the adaptability of $\pi_{0.5}$ equipped with our proposed One-Shot Feature Linear Adaptation.
}\label{fig:novel_camera_viewpoint}
\end{figure}

To enable systematic evaluation, we construct the Libero-V (Visual) benchmark, which integrates varying levels of viewpoint and visual perturbations from Libero-Plus~\cite{fei2025liberoplusindepthrobustnessanalysis}. Libero-V introduces four controlled perturbation types—camera viewpoint~\cite{wilcox2025adapt3radaptive3dscene}, lighting, background texture, and visual noise—providing a unified framework for assessing robustness under diverse distribution shifts. Our empirical findings reveal that pretrained VLA models inherently encode latent robustness, which can be efficiently activated through minimal adaptation rather than large-scale retraining. This perspective highlights that achieving robust embodied perception does not necessarily require more data or complex architectures but rather efficient mechanisms for perturbation-invariant feature adaptation (See Fig.~\ref{fig:novel_camera_viewpoint}).

Our main contributions are summarized as follows:

\begin{itemize}
\item We reassess the robustness of pretrained VLA models and show that performance degradation under visual perturbations primarily arises from representation shifts within Spatial Modeling. 

\item We propose a unified one-shot robustness adaptation framework, consisting of Feature Token Modulation (FTM) and Feature Linear Adaptation (FLA), which efficiently recalibrates visual representations with minimal parameter updates.

\item Our approach achieves state-of-the-art generalization to unseen camera viewpoints and visual perturbations on the Libero benchmark, revealing the untapped robustness potential within existing VLA models.
\end{itemize}

%% file: sec/2_related.tex
\section{Related Work}
\subsection{Vision-Language-Action Models}
VLA models aim to unify visual perception, language grounding, and action generation within a single policy framework \cite{zhan2026e0enhancinggeneralizationfinegrained,song2026learning,song2025physical}. 
By extending pretrained vision--language foundations to the robotic domain, recent VLAs such as RT-2~\cite{brohan2023rt2visionlanguageactionmodelstransfer}, Pi0~\cite{black2024pi0visionlanguageactionflowmodel}, PaLM-E~\cite{driess2023palme}, Octo~\cite{octomodelteam2024octoopensourcegeneralistrobot}, and OpenVLA~\cite{kim2024openvlaopensourcevisionlanguageactionmodel} demonstrate strong in-distribution generalization across diverse manipulation tasks. 
These systems leverage large-scale embodiment datasets~\cite{khazatsky2025droidlargescaleinthewildrobot, embodimentcollaboration2025openxembodimentroboticlearning} to integrate varied visual inputs, language instructions, and proprioceptive states, along with noisy actions, to generate a shared action policy that enables robots to perform versatile behaviors.

Despite these advances, recent analyses reveal that state-of-the-art VLAs, even after being pre-trained and subsequently fine-tuned on datasets, remain surprisingly brittle under distribution shifts~\cite{zhou2025liberoprorobustfairevaluation, fei2025liberoplusindepthrobustnessanalysis}. 
Performance degrades sharply when models are exposed to unseen camera viewpoints, novel lighting conditions, or scene noise perturbations---a limitation that persists even in models trained on millions of demonstrations. 
Understanding and addressing this brittleness is therefore essential for deploying VLA systems in real-world environments, where the visual environments are significantly more complex than Libero settings.

\subsection{Viewpoint Robustness \& Embedding Stability}
Robustness to visual perturbations is a long-standing challenge in improving VLA performance \cite{zhou2026tagtargetagnosticguidancestable}. 
Prior work has explored data-centric strategies to improve viewpoint invariance by increasing visual diversity through domain randomization, multi-view demonstrations, or large-scale visual perturbation datasets~\cite{fei2025liberoplusindepthrobustnessanalysis, zhou2025liberoprorobustfairevaluation}. 
While effective, these approaches require extensive data collection or annotation, which limits their scalability.

Another line of research focuses on representation-centric robustness by enforcing geometric consistency\cite{wilcox2025adapt3radaptive3dscene, yang2025fp33dfoundationpolicy, abouzeid2025geoawarevlaimplicitgeometryaware}. Recent studies further indicate that viewpoint shifts can induce significant drifts in the embedding space of pretrained ViT encoders~\cite{xie2023decomposinggeneralizationgapimitation}, disrupting the coordination between perception and downstream policy modules. 

Motivated by these findings, we investigate whether correcting such embedding misalignment---without modifying the overall VLA architecture or collecting additional data---is sufficient to restore viewpoint robustness.

\subsection{Parameter-Efficient Adaptation}
Parameter-efficient finetuning (PEFT) has emerged as an effective alternative to full-model adaptation, enabling models to specialize to new tasks or domains with minimal trainable parameters. 
Representative approaches include adapters~\cite{houlsby2019parameterefficienttransferlearningnlp}, LoRA~\cite{hu2022lora}, Prefix-Tuning~\cite{li2021prefixtuningoptimizingcontinuousprompts}, and BitFit~\cite{zaken2022bitfitsimpleparameterefficientfinetuning}, all of which introduce small learnable modules while keeping the majority of the backbone frozen. 
These techniques have been widely adopted in large language models and multimodal foundations due to their efficiency and stability.

In embodied AI, PEFT has primarily been applied to language or policy components, such as lightweight updates to decision transformers or instruction encoders~\cite{black2024pi0visionlanguageactionflowmodel, kim2024openvlaopensourcevisionlanguageactionmodel, octomodelteam2024octoopensourcegeneralistrobot}.
When addressing viewpoint shifts, most existing methods rarely solve the problem by directly finetuning the visual module itself.
Approaches that perform lightweight and targeted adaptation of the visual pathway within VLA models remain particularly scarce.

In contrast to these approaches, which often adapt the entire VLA stack, our work focuses on parameter-efficient adaptations of the visual pathway, investigating whether lightweight adaptation of vision tokens or the vision encoder alone is sufficient to recover robustness.

%% file: sec/3_method.tex
\section{Method}
\label{method}

Prior works typically address brittleness to visual perturbations by finetuning large portions of the VLA or retraining the entire visuomotor policy, as illustrated in Fig.~\ref{fig:solution} (a, b). 
In contrast, we argue that global adaptation is unnecessary: shifts caused by visual perturbations mainly induce systematic drifts in the visual embedding space—disrupting the coordination between the visual encoder and the VLM head rather than exposing a lack of visuomotor capacity.
This observation suggests adapting only the visual module through lightweight mechanisms—either token modulation or a LoRA-based update—instead of globally finetuning the VLA.

\subsection{Preliminaries}
\label{sec:preliminaries}

\paragraph{Original $\pi_{0.5}$ formulation.} 
We adopt $\pi_{0.5}$~\cite{intelligence2025pi05visionlanguageactionmodelopenworld} as our base vision-language-action (VLA) policy. The policy operates on language and visual observations and autoregressively produces discrete action tokens at each control timestep. Concretely, at time $t$ the agent observes a visual input $v_t$ (an image or multi-view image set) and a language instruction/goal $l$, which together we denote by the observation $o_t=(v_t,l)$. The policy parametrized by $\theta$ defines a distribution over an action token sequence $a_{1:T}$:
\begin{equation}
    P_\theta(a_{1:T}\mid o_{1:T}) = \prod_{t=1}^{T} P_\theta(a_t \mid a_{<t}, o_{\le t}),
\end{equation}
where $a_t$ is represented as a short sequence of discrete tokens encoding robot control (e.g., position, orientation, and gripper commands).

Architecturally, $\pi_{0.5}$ decomposes into three conceptual modules: a visual encoder $f_v(\cdot)$ that maps images to token embeddings, a language encoder $f_\ell(\cdot)$ that maps text to embeddings, and a multimodal transformer decoder $g(\cdot)$, equipped with separate expert weights, that autoregressively predicts action tokens conditioned on the fused embeddings. Denoting the visual-token sequence by $\mathbf{z} = f_v(v)$ and the language embedding by $\ell=f_\ell(l)$, the decoder operates on the concatenated multimodal sequence:
\begin{equation}
    \hat{a}_t \sim g\big(a_{<t};\; [\mathbf{z};\ \ell]\big).
\end{equation}

\paragraph{Notation for adaptation.} 
In the remainder of the paper we distinguish between (i) the pretrained VLA parameters $\theta$ and (ii) lightweight adaptation parameters $\phi$ that we introduce to improve robustness under visual perturbations. Adaptation interventions will be applied only to the visual-side representations (i.e., the output of $f_v$ or to linear layers inside the visual encoder) while keeping the decoder $g$ and language encoder $f_\ell$ fixed. Formally, after adaptation the prediction distribution becomes
\begin{equation}
    P_{\theta,\phi}(a_t\mid a_{<t}, o_{\le t}) \;=\; g\big(a_{<t};\; [\,\mathcal{A}_\phi(f_v(v));\ \ell\,]\big),
\end{equation}
where $\mathcal{A}_\phi(\cdot)$ denotes a lightweight, learnable transformation applied to visual tokens (for example, the Feature Token Modulation described in Sec.~\ref{Feature Token Modulation} or the Feature Linear Adaptation in Sec.~\ref{Features Linear Adaptation}).

\begin{figure}
    \centering
    \includegraphics[width=\linewidth]{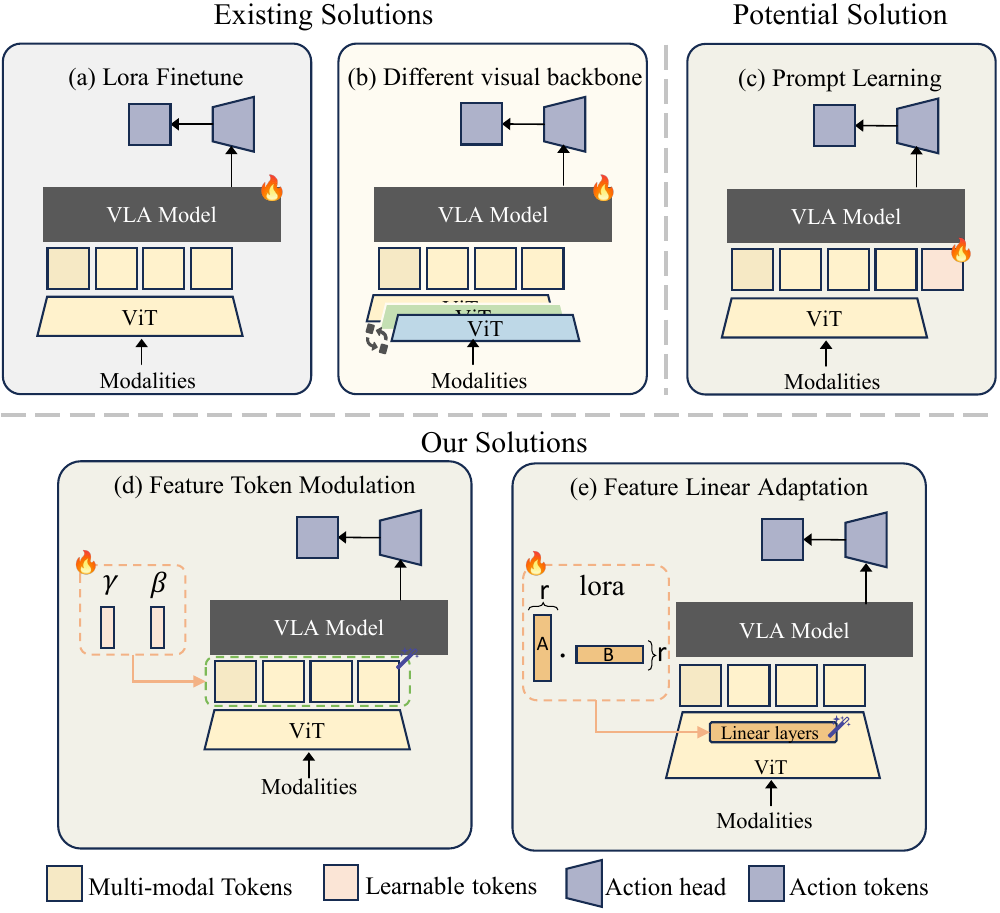}
    \caption{\textbf{Comparison of methods for adapting to new visual perturbations in VLA models.}
Panels (a) and (b) show existing finetuning-based approaches for adapting VLA models to new visual features. Panel (c) illustrates a potential meta-learning strategy using concatenated learnable prompts. Panels (d) and (e) present our proposed methods—Feature Token Modulation and Feature Linear Adaptation.
}
    \label{fig:solution}
\end{figure}

\subsection{Approaches for Robustness Adaptation}
\label{sec:approaches}

To mitigate brittleness to visual perturbations, several adaptation strategies have been explored in prior work, as illustrated in Fig.~\ref{fig:solution}:
\begin{enumerate}[(a)]
\item \textbf{LoRA Fine-tuning.} A common solution is to perform full or parameter-efficient fine-tuning of the VLA backbone $\theta$, such as LoRA-based updates~(Fig.~\ref{fig:solution}a), which adjust intermediate feature mappings while preserving pretrained weights.
This post-training technique is widely adopted in large language and vision-language models~\cite{geminiteam2025geminifamilyhighlycapable,openai2024gpt4technicalreport}, and has been further incorporated into the training paradigm of VLAs.

\item \textbf{Alternative Visual Backbones.}
Another line of work replaces the visual encoder entirely with a more robust backbone~\cite{abouzeid2025geoawarevlaimplicitgeometryaware,wilcox2025adapt3radaptive3dscene}~(Fig.~\ref{fig:solution}b), aiming to improve invariance through stronger visual representations. Such replacement, however, necessitates additional training of the VLA to reestablish consistency between the updated visual features and its language-conditioned action head.

\item \textbf{Prompt-based Adaptation.}
Prompt-based adaptation~\cite{zhou2022coop,zhou2022cocoop}~(Fig.~\ref{fig:solution}c) introduces learnable tokens concatenated directly to multi-modal embeddings, conditioning the model on perturbation contexts and enabling partial alignment across domains.
\end{enumerate}

While these approaches can improve performance under distribution shifts, they often incur substantial parameter overhead or require extensive retraining. In contrast, our method focuses on lightweight, modular adaptation of the pretrained visual backbone without altering its architecture or data distribution.

In contrast, we pursue a more lightweight and modular solution that adapts the pretrained VLA without altering its architecture, in a one-shot manner.  
As shown in Fig.~\ref{fig:solution} (d–e), our framework introduces two simple yet effective mechanisms: \textbf{Feature Token Modulation (FTM)}, which globally adjusts visual token embeddings through learnable affine transformations, and \textbf{Feature Linear Adaptation (FLA)}, a low-rank update to the visual backbone for deeper feature realignment.  
Together, these mechanisms form a unified one-shot robustness adaptation framework that efficiently restores viewpoint and perturbation robustness with minimal parameter updates.

\begin{figure*}[t]
    \centering
    \begin{subfigure}[t]{0.24\textwidth}
        \centering
        \includegraphics[width=\linewidth]{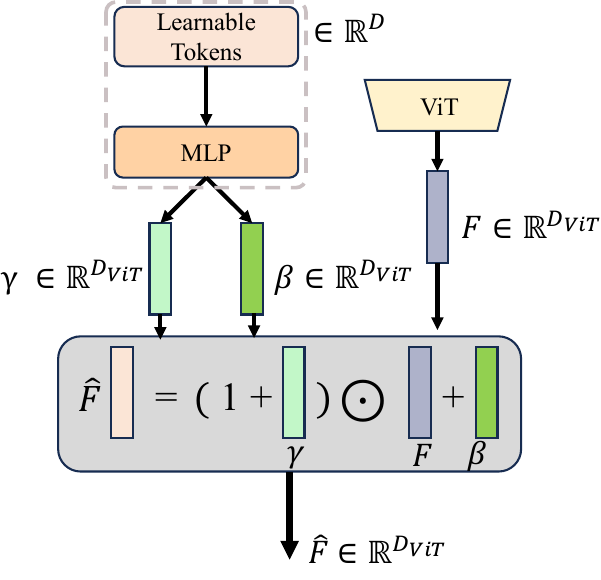}
        \caption{Feature Token Modulation.}
        \label{fig:FiLM_a}
    \end{subfigure}
    \hfill
    \begin{subfigure}[t]{0.74\textwidth}
        \centering
        \includegraphics[width=\linewidth]{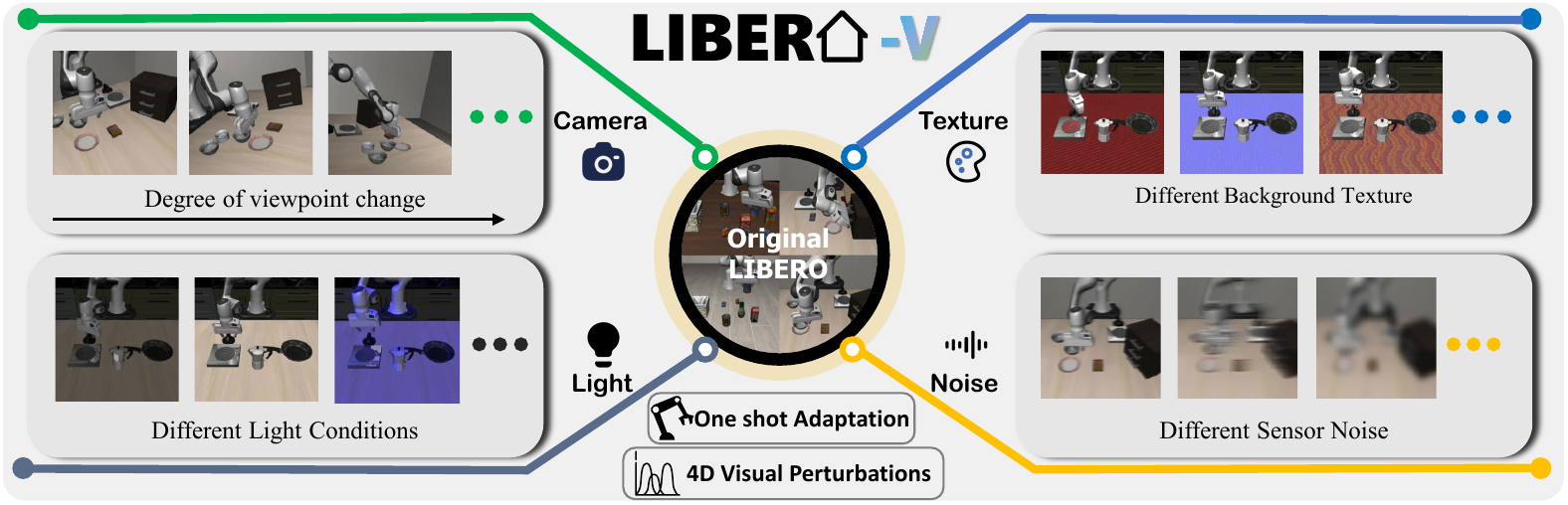}
        \caption{Visualization of Libero-V(Visual) benchmark}
        \label{fig:FiLM_b}
    \end{subfigure}
    \caption{
\textbf{Overview of feature adaptation and benchmark design.}
(a) Illustration of the proposed Feature Token Modulation mechanism. The dashed box denotes components used only during training and removed at inference. 
(b) The LIBERO-V (Visual) benchmark, created by combining multiple levels of viewpoint variation~\cite{wilcox2025adapt3radaptive3dscene} with visual-perturbation tasks from LIBERO-Plus~\cite{fei2025liberoplusindepthrobustnessanalysis}. It covers four perturbation types: camera viewpoint, lighting, background texture, and image noise.
}
    \label{fig:FTM_&_Libero_V}
\end{figure*}

\begin{table*}[t]
  \centering
  \caption{\textbf{Success rates on the LIBERO benchmark under novel camera viewpoints.}
We report Success Rate (SR) across all unseen viewpoints~\cite{wilcox2025adapt3radaptive3dscene} in the LIBERO suites~\cite{liu2023liberobenchmarkingknowledgetransfer}. GeoAware-VLA~\cite{abouzeid2025geoawarevlaimplicitgeometryaware} replaces its visual backbone with VGGT and trains policies from scratch (results taken from the original paper). The two OpenVLA variants include OpenVLA-OFT (zero-shot) and OpenVLA-OFT-m, which is finetuned on LIBERO-Plus~\cite{fei2025liberoplusindepthrobustnessanalysis} using large-scale multi-view data. We additionally evaluate one-shot adaptations applied to pretrained $\pi_0$ and $\pi_{0.5}$, including LoRA, Feature Token Modulation (FTM), and Feature Linear Adaptation (FLA).
}
    \begin{tabular}{l|c|c|c|c|c}
    \toprule
    \multirow{3}[4]{*}{Model} & \multicolumn{4}{c|}{Libero Suites (Across Novel Camera)} & Average \\
\cmidrule{2-6}          & Spatial & Object & Goal  & Long  & Across Suites \\
          & SR(\%) & SR(\%) & SR(\%) & SR(\%) & SR(\%) \\
    \midrule
    GeoAware-VLA(GeoAware VQ-BeT) & 54.3  & 99.0  & 85.7  & 72.7  & 77.9 \\
    GeoAware-VLA(GeoAware BAKU) & 94.3  & 98.0    & 90.7  & 47.3  & 82.6 \\
    OpenVLA-OFT & 90.0  & 15.7  & 81.7  & 13.7  & 50.3 \\
    OpenVLA-OFT-m(Fine-Tuned on Libero-Plus~\cite{fei2025liberoplusindepthrobustnessanalysis}) & 67.1  & 72.4  & 72.0  & 49.2  & 65.2 \\
    $\pi_0$(One-Shot Lora Fine-tuned) & 95.4  & 96.6  & 79.1  & 63.2  & 83.6 \\
    $\pi_{0.5}$(One-Shot Lora Fine-tuned) & \textbf{97.2} & 95.3  & 91.4  & 77.1  & 90.3 \\
    \midrule
    $\pi_{0.5}$(One Shot FTM)(ours) & 95.4  & 95.5  & 87.7  & 70.0  & 87.1 \\
    $\pi_{0.5}$(One-Shot FLA)(ours) & 96.5  & \textbf{97.5} & \textbf{91.7} & \textbf{77.5} & \textbf{90.8} \\
    \bottomrule
    \end{tabular}
  \label{tab:libero}
\end{table*}

\begin{table*}[t]
  \centering
  \caption{\textbf{Success Rate (SR) on Novel View Across all four Libero sub-suites.}}
  \resizebox{0.7\linewidth}{!}{
    \begin{tabular}{l|c|c|c|c}
    \toprule
    \multirow{3}[4]{*}{Model} & \multicolumn{4}{c}{Libero Suites (Novel Camera)} \\
\cmidrule{2-5}          & Small & Medium & Large & Average \\
          & SR(\%) & SR(\%) & SR(\%) & SR(\%) \\
    \midrule
    GeoAware-VLA & 88.0    & 74.5  & 71.2  & 77.9 \\
    OpenVLA-oft-m(fine-tuned on Libero-Plus) & 88.1  & 65.5  & 42.1  & 65.2 \\
    $\pi_0$(One-Shot Lora Fine-tuned) & 88.3  & 83    & 79.5  & 83.6 \\
    $\pi_{0.5}$(One-Shot Lora Fine-tuned) & \textbf{94.8} & \textbf{90.5} & 85.6  & 90.3 \\
    \midrule
    Ours  & 94.6  & 90.0    & \textbf{87.9} & \textbf{90.8} \\
    \bottomrule
    \end{tabular}
  }
  \label{tab:novelview}
\end{table*}

\subsection{Feature Token Modulation}
\label{Feature Token Modulation}
Inspired by the principles of meta-learning~\cite{finn2017modelagnosticmetalearningfastadaptation}, we believe that robust behavior can be generated through effective adaptation rather than extensive retraining. To implement this idea, we first apply a lightweight modulation on the output token embeddings $F \in \mathbb{R}^{N \times D_{\mathrm{ViT}}}$ extracted from the camera image, introducing two learnable parameters $\gamma, \beta \in \mathbb{R}^{D_{\mathrm{ViT}}}$ to scale and shift the feature distribution:
\begin{equation}
    \hat{F} = (1 + \gamma) \odot F + \beta.
\end{equation}

This form of modulation can be viewed as an affine correction to the visual embedding space, effectively re-centering and re-scaling feature dimensions distorted by viewpoint shifts, while keeping the VLA backbone frozen, as illustrated in Fig.~\ref{fig:solution} (d) and Fig.~\ref{fig:FTM_&_Libero_V} (a). 
In practice, unlike conditional modulation methods that generate input-dependent parameters, our $\gamma$ and $\beta$ are globally applied yet learnable, jointly optimized during adaptation while keeping the backbone frozen. This design introduces only $2D_{\mathrm{ViT}}$ trainable parameters in total.

We treat this modulation as a controlled probe to examine whether viewpoint brittleness primarily originates from embedding misalignment rather than insufficient model capacity.
If the hypothesis holds, adjusting only $(\gamma,\beta)$ should be sufficient to partially restore viewpoint alignment without global finetuning; the empirical evidence is presented in Table~\ref{tab:libero}.

\subsection{Feature Linear Adaptation}
\label{Features Linear Adaptation}
Given the effectiveness of Feature Token Modulation (Sec.~\ref{Feature Token Modulation}), we further explore whether viewpoint-induced misalignment can also be corrected by directly adapting the visual backbone itself. 
Specifically, instead of adjusting the output tokens, we apply a parameter-efficient LoRA\cite{hu2022lora} update to the linear layers inside the ViT encoder, as illustrated in Fig.~\ref{fig:solution} (e).

Let $h = W x$ denote a linear transformation within the ViT, where $W \in \mathbb{R}^{d_{\text{out}} \times d_{\text{in}}}$ is frozen during adaptation. 
LoRA introduces a low-rank decomposition,
\begin{equation}
    W' = W + \Delta W, \quad 
    \Delta W = B A,
\end{equation}
where $A \in \mathbb{R}^{r \times d_{\text{in}}}$ and $B \in \mathbb{R}^{d_{\text{out}} \times r}$ with $r \ll \min(d_{\text{in}}, d_{\text{out}})$. 
Only $(A,B)$ are trainable, allowing ViT to adjust its feature extraction with minimal parameter overhead.

This linear adaptation serves as a second minimal intervention on the Spatial Modeling, allowing us to test whether adjusting internal layers yields effects comparable to, or slightly better than, token-level modulation. 
Empirical comparisons are reported in Table~\ref{tab:libero} and Table~\ref{tab:liberov}.

\begin{figure*}
    \centering
    \includegraphics[width=\linewidth]{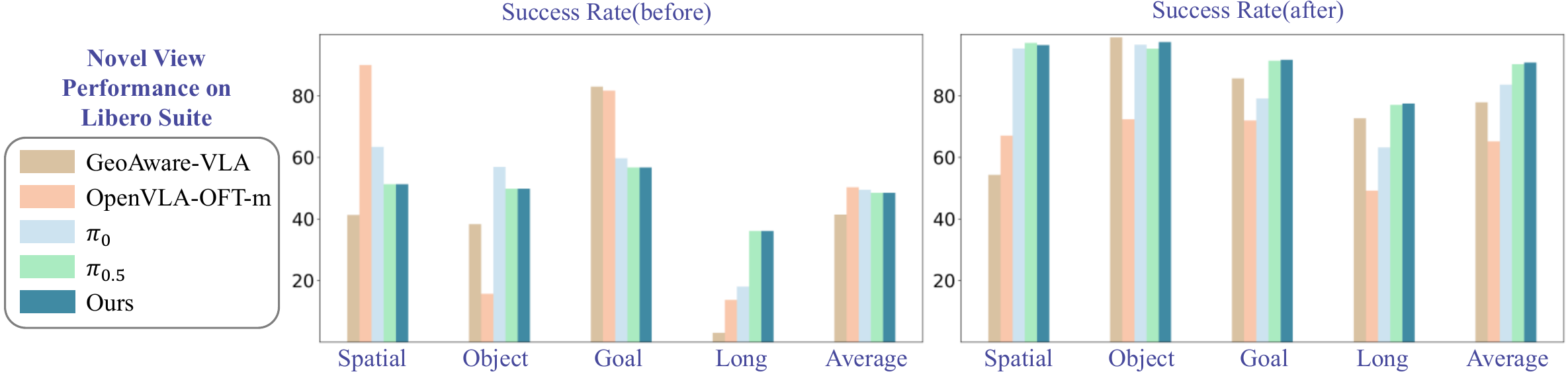}
    \caption{\textbf{Success rates before and after adaptation on the LIBERO benchmark under novel camera viewpoints.}
We report Success Rate (SR) across all unseen viewpoints~\cite{wilcox2025adapt3radaptive3dscene} in the LIBERO suites~\cite{liu2023liberobenchmarkingknowledgetransfer}. “Before’’ corresponds to the zero-shot performance of pretrained policies without any adaptation.
}
    \label{fig:successrate}
    \vspace{-6mm}
\end{figure*}

\begin{table*}[htbp]
  \centering
  \caption{\textbf{Success rates on the LIBERO-V (Visual) benchmark.}
The benchmark evaluates robustness across four visual perturbations: camera viewpoint, lighting, background texture, and noise. Results marked with $^*$ are taken from~\cite{fei2025liberoplusindepthrobustnessanalysis}.
}
    \begin{tabular}{lccccc}
    \toprule
    \multirow{2}[4]{*}{Model} & \multicolumn{4}{c}{Visual Perturbations} & Average \\
\cmidrule{2-6}          & camera & light & texture & noise & Across Subsets \\
    \midrule
    OpenVLA-OFT & 50.3  & 92.7  & 92.3  & 78.6$^*$ & 78.5 \\
    OpenVLA-OFT-m(Fine-Tuned on Libero-Plus~\cite{fei2025liberoplusindepthrobustnessanalysis})) & 65.2  & 94.9  & 93.9  & 89.9$^*$ & 86.0 \\
    $\pi_{0}$(One-Shot Lora Fine-tuned) & 83.6  & 90.7  & 86.7  & 87.0$^*$ & 87.0 \\
    $\pi_{0.5}$(Zero-Shot) & 48.5  & 96.2  & 96.0  & 93.5  & 83.6 \\
    $\pi_{0.5}$(One-Shot Lora Fine-tuned) & 90.3  & 96.5  & \textbf{97.2} & 94.5  & 94.6 \\
    \midrule
    $\pi_{0.5}$(One Shot FTM)(ours) & 87.1  & 96.0  & 96.0  & 93.6  & 90.5 \\
    $\pi_{0.5}$(One-Shot FLA)(ours) & \textbf{90.8} & \textbf{96.8} & 97.1  & \textbf{94.6} & \textbf{94.8} \\
    \bottomrule
    \end{tabular}
  \label{tab:liberov}
\end{table*}

%% file: sec/4_experiment.tex
\section{Experiment}
\label{experiment}

\subsection{Libero-V Benchmarks}
To systematically evaluate the generalization of current VLA policies for visual perturbations, we construct the Libero-V (Visual) benchmark, which extends the original LIBERO benchmark to incorporate diverse visual perturbations.
The original LIBERO benchmark \cite{liu2023liberobenchmarkingknowledgetransfer} assesses knowledge transfer in multitask and lifelong robot learning for manipulation tasks. It consists of four task suites, and each suite contains 10 tasks.

Following~\cite{wilcox2025adapt3radaptive3dscene} we evaluate each policy with 50 trials per task on novel camera viewpoints. In addition, Libero-V integrates visual perturbations from Libero-Plus \cite{fei2025liberoplusindepthrobustnessanalysis}, covering four types of visual variation:
(1) different camera viewpoints,
(2) lighting changes,
(3) background and tabletop textures, and
(4) sensor-induced image noise.
These perturbations simulate realistic distribution shifts and allow for a unified evaluation of robustness under visual domain changes.
Importantly, our method performs one-shot adaptation — it requires only a single human demonstration per task to adapt to the new visual domain.

\subsection{Baselines}
We adopt $\pi_{0.5}$ \cite{intelligence2025pi05visionlanguageactionmodelopenworld} as our base vision-language-action (VLA) policy, which operates on visual and language inputs to generate action sequences autoregressively.
We compare our proposed one-shot adaptation methods—Feature Token Modulation (FTM) and Feature Linear Adaptation (FLA)—against a range of strong baselines, including:
\begin{itemize}
    \item GeoAware-VLA \cite{abouzeid2025geoawarevlaimplicitgeometryaware}, which replaces the visual backbone with a geometry-aware encoder and retrains the policy from scratch.
    \item OpenVLA-OFT and OpenVLA-OFT-m \cite{kim2025finetuningvisionlanguageactionmodelsoptimizing}, where the latter is fine-tuned on Libero-Plus to improve viewpoint robustness.
    \item LoRA-based One-Shot Fine-tuning on both $\pi_0$ and $\pi_{0.5}$, which updates the policy parameters through low-rank adaptation.
    \item Prompt Learning, which introduces additional learnable tokens concatenated to multimodal embeddings.
\end{itemize}

Our proposed FTM and FLA differ in adaptation scope:
FTM applies a global affine modulation on the visual token embeddings with only 4K parameters, while FLA performs a low-rank linear adaptation within the ViT encoder using 4.7M parameters.

\subsection{Implementation details}
All experiments are conducted on a single NVIDIA A100 GPU with 80GB of VRAM.
In Feature Token Modulation (FTM), the learnable parameters $\gamma, \beta \in \mathbb{R}^{D_{\mathrm{ViT}}}$ correspond to the hidden dimension of the visual encoder in $\pi_{0.5}$, where $D_{\mathrm{ViT}} = 2048$. As a result, FTM introduces only 4K trainable parameters in total.
For Feature Linear Adaptation (FLA), we fine-tune only the linear layers within the SigLIP visual backbone~\cite{tschannen2025siglip2multilingualvisionlanguage} of $\pi_{0.5}$.
To ensure a fair comparison, both LoRA and FLA one-shot adaptation procedures are trained for 2000 steps using the Adam optimizer.
Across all experiments, we use a batch size of 32, and the policy receives observations from two cameras (a wrist camera and a third-person camera) to provide complementary visual information. We follow ~\cite{intelligence2025pi05visionlanguageactionmodelopenworld} to implement the Lora fine-tuning for baselines~\cite{black2024pi0visionlanguageactionflowmodel, intelligence2025pi05visionlanguageactionmodelopenworld}.

\subsection{Results and Analysis}

\paragraph{Robustness to Novel Camera Viewpoints.}
Table~\ref{tab:libero} summarizes the performance of all methods on the LIBERO benchmark suites across novel camera viewpoints, while Fig.~\ref{fig:successrate} visualizes the success rate before and after distinct adaptations.
Our method achieves consistent improvements across all four LIBERO suites (Spatial, Object, Goal, Long), demonstrating its effectiveness in mitigating the brittleness of pretrained VLAs under viewpoint shifts.

Specifically, the Feature Token Modulation (FTM) variant achieves an average success rate of 87.1\%, surpassing the $\pi_{0}$ baseline (83.6\%) despite updating only a pair of global affine parameters. 
This result highlights a key insight of our study: under viewpoint perturbations, the visual tokens produced by Spatial Modeling become misaligned with the Physical Modeling components, yet even minimal token-level modulation is sufficient to substantially recover robustness.

Building on this observation, the Feature Linear Adaptation (FLA) variant further improves performance to 90.8\%, surpassing the strong $\pi_{0.5}$ One-Shot LoRA baseline (90.3\%) while using less than 1\% of its parameters (4.7M vs.\ 467M). 
That a targeted, lightweight update to Spatial Modeling alone can outperform full-model LoRA fine-tuning provides compelling evidence that viewpoint brittleness originates primarily within the spatial representation, rather than the visuomotor policy.

Furthermore, Fig.~\ref{fig:novel_camera_viewpoint} and Table~\ref{tab:novelview} show that our model maintains stable performance as the degree of viewpoint perturbation increases (Small → Large), indicating strong invariance to spatial transformations and validating our hypothesis that latent robustness can be recovered through efficient adaptation instead of retraining or large-scale data augmentation.

\paragraph{Robustness to Visual Perturbations (Libero-V).}
We further evaluate all methods on the Libero-V (Visual) benchmark (Table~\ref{tab:liberov}), which introduces four controlled perturbations: camera, lighting, texture, and noise.
Across all visual domains, our proposed one-shot adaptations deliver robust and consistent performance.

In particular, FLA achieves an average success rate of 94.8\%, matching or surpassing LoRA fine-tuning while being two orders of magnitude more parameter-efficient.
Despite the already strong performance of the One-Shot LoRA baseline, our FLA slightly outperforms LoRA on the camera perturbation subset (90.8\% vs. 90.3\%) and also maintains excellent performance on lighting, texture, and noise perturbations, indicating better cross-perturbation generalization.
Meanwhile, FTM, with only 4K parameters, already recovers most of the viewpoint robustness (87.1\%), confirming that even simple token-level spatial adaptation can reactivate latent invariances within pretrained VLA representations.

These results reinforce our central claim: robustness under visual perturbations is limited primarily by Spatial Modeling, and targeted lightweight adaptation is sufficient to reactivate the latent robustness of pretrained VLA models—without large-scale retraining or extensive domain randomization.

\begin{table}[t]
  \centering
  \caption{\textbf{Parameter count and accuracy on the LIBERO-V (Visual) benchmark.}
}
  \resizebox{\linewidth}{!}{
    \begin{tabular}{l|cc}
    \toprule
    Methods & \#Params (M) & Libero-V(\%) \\
    \midrule
    $\pi_{0}$(One-Shot Lora Fine-tuned) & 468.04 & 87 \\
    $\pi_{0.5}$(Prompt Learning) & 0.131  & 75.1 \\
    $\pi_{0.5}$(One-Shot Lora Fine-tuned) & 466.96 & 94.6 \\
    $\pi_{0.5}$(One Shot FTM)(ours) & \textbf{0.004} & 90.5 \\
    $\pi_{0.5}$(One-Shot FLA)(ours) & 4.7  & \textbf{94.8} \\
    \bottomrule
    \end{tabular}
    }
  \label{tab:parameter-efficiency}
\end{table}

\paragraph{Ablation on Adaptation Rank and Baselines.}
We conduct ablation studies on both the adaptation rank of the Feature Linear Adaptation (FLA) module and the choice of baselines. Table~\ref{tab:ablation} reports results on the Libero suites under novel camera viewpoints.
As shown in Table~\ref{tab:ablation}, increasing the adaptation rank from 16 to 32 results in a modest improvement in success rate (from 90.8\% to 91.2\%). This is expected, since a higher rank corresponds to a larger number of learnable parameters in the visual backbone, enabling finer-grained feature realignment and improved robustness under complex visual perturbations.

In addition, Table~\ref{tab:ablation} compares FLA with LoRA using two different pretrained policies, $\pi_{0}$ and $\pi_{0.5}$. For both baselines, applying FLA (rank=16) achieves accuracy comparable to—or slightly higher than—LoRA, while requiring only 4.7M trainable parameters compared to LoRA’s 467M. Notably, on $\pi_{0.5}$, FLA reaches 90.8\% success rate with a 99× reduction in parameters, highlighting its effectiveness and efficiency across different base models.

\paragraph{Parameter Efficiency and Adaptation Strategy.}
Table~\ref{tab:parameter-efficiency} compares all adaptation mechanisms in terms of trainable parameter count and Libero-V accuracy.
Prompt Learning (Fig.~\ref{fig:solution}c) achieves modest improvement (75.1\%) with negligible parameters (0.13M), but its shallow conditioning is insufficient for full feature realignment.
In contrast, FTM (0.004M) achieves substantially higher robustness (90.5\%), showing that direct modulation of visual tokens is a more effective adaptation pathway.
FLA further extends this idea by introducing low-rank structure into the ViT encoder, achieving 94.8\% success rate with 4.7M parameters—a 99× reduction compared to LoRA while maintaining nearly identical accuracy.

This efficiency-performance trade-off demonstrates the scalability and practicality of our one-shot adaptation framework.
It supports our overarching conclusion that recovering robustness in embodied VLA models does not require more data or complex architectures, but can be accomplished through simple, efficient, and well-targeted feature adaptation.

\begin{table}[t]
  \centering
  \caption{\textbf{Efficiency of our FLA adaptation.}
FLA (rank = 16) achieves a 90.8\% success rate using only 4.7M parameters, compared to LoRA’s 467M parameters at 90.3\%.
}
  \resizebox{\linewidth}{!}{
    \begin{tabular}{lcc}
    \toprule
    Model & \multicolumn{1}{l}{\#Params} & \multicolumn{1}{l}{Success Rate(\%)} \\
    \midrule
    $\pi_{0}$(One-Shot Lora Fine-tuned) & 468M  & 83.6 \\
    $\pi_{0.5}$(One-Shot Lora Fine-tuned) & 467M  & 90.3 \\
    \midrule
    $\pi_{0}$(One-Shot FLA)(ours, rank=16) & \textbf{4.7M} & 84.0 \\
    $\pi_{0.5}$(One-Shot FLA)(ours, rank=16) & \textbf{4.7M} & 90.8 \\
    $\pi_{0.5}$(One-Shot FLA)(ours, rank=32) & 9.4M  & \textbf{91.2} \\
    \bottomrule
    \end{tabular}}
  \label{tab:ablation}
\end{table}

\section{Conclusion}
We show that this brittleness primarily arises from misalignment in Spatial Modeling rather than limitations in Physical Modeling, indicating that the visuomotor policy itself remains fundamentally capable.
We introduced two lightweight mechanisms—Feature Token Modulation (FTM) and Feature Linear Adaptation (FLA)—that recalibrate visual representations with minimal parameters.
Across the Libero-V benchmarks, our framework achieves state-of-the-art generalization under novel camera viewpoints and diverse visual perturbations, while being up to 99× more parameter-efficient than LoRA fine-tuning.
These results suggest that enhancing embodied visual robustness does not require larger models or more data; instead, efficient, targeted adaptation can unlock the inherent generalization potential of existing VLA architectures.

\section{Acknowledgments}

This work was supported in part by National High-Level Young Talent Program (Grant 2025HY00260104), in part by the Fundamental Research Funds for Higher Education Institutions allocated to Sun Yat-sen University (Grants 25hytd007 and 2025RGZN009), in part by the Guangdong Provincial High-Level Young Talent Program (Grant 2025HYSPT0707), in part by the Tuoyuan Grant (HT-99982025-0564), in part by the Faculty Start-up Research Fund (Grant 67000-12255002), and in part by the Huawei Strategic Research Institute Talent Fund.

%% file: sec/X_suppl.tex
\clearpage
\setcounter{page}{1}
\maketitlesupplementary

\theoremstyle{plain}
\newtheorem{theorem}{Theorem}

\newcommand{\eps}{\epsilon}
\newcommand{\D}{\mathcal{D}}
\newcommand{\E}{\mathbb{E}}
\newcommand{\R}{\mathbb{R}}
\newcommand{\dTV}{d_{\mathrm{TV}}}

\section{Real-World Experiments}
\label{sec:real_world}

To validate the practical efficacy of our method beyond simulation, we established a physical tabletop manipulation environment centered around a Franka Emika Panda robot. This section details our experimental setup, data collection pipeline, task design, and the deployment of our Feature Linear Adaptation (FLA) method. Please refer to the accompanying video in the supplementary material for visualization.

\subsection{Experimental Setup and Data Collection}
The robot setup features a 7-DoF Franka Emika Panda arm equipped with a standard parallel-jaw gripper. The system operates within an 8-dimensional configuration and action space (7 joint positions + 1 gripper state). To capture comprehensive visual information, we employ two fixed camera views: a \textbf{3rd-person static camera} providing a global view of the workspace, and a \textbf{wrist-mounted camera} providing egocentric observations. The physical configuration is illustrated in Fig.~\ref{fig:real_setup}.

For data collection, we utilized \textbf{GELLO}~\cite{wu2023gello}, a general, low-cost, and intuitive teleoperation framework. This allowed for high-quality human demonstrations to be collected efficiently. We utilized the $\pi_{0.5}$ architecture as our base policy.

\subsection{Evaluation Tasks}
We designed five distinct real-world manipulation tasks ranging from basic object interaction to complex articulated object manipulation. These tasks were selected to rigorously evaluate the model's spatial grounding capabilities and precision under novel viewpoints:

\begin{enumerate}    
    \item \textbf{``Pick up the red block and stack it on the green block''}: A multi-stage task that assesses precision and depth perception. Successfully stacking the block requires the model to maintain geometric consistency throughout the trajectory, a challenge often exacerbated by viewpoint shifts.

    \item \textbf{``Pull out the top drawer''}: This task involves manipulating a prismatic joint with frictional resistance. It evaluates the policy's robustness in maintaining contact and exerting force in the correct direction under a novel visual perspective.

    \item \textbf{``Press the green button''}: A fine-grained control task requiring high spatial precision. The target is small relative to the scene, making it highly sensitive to the misalignment of spatial features usually caused by camera viewpoint changes.

    \item \textbf{``Pick up the red block on the table''}: This task evaluates the fundamental spatial grounding capability of the model. It requires the agent to correctly identify the target object based on color semantics and accurately estimate its 3D position from the new camera viewpoint.

    \item \textbf{``Close the microwave oven door''}: This task tests the model's ability to interact with articulated objects. It requires understanding the kinematic constraints of the door and executing a precise push motion.
    
\end{enumerate}

\begin{figure}
    \centering
    \includegraphics[width=\linewidth]{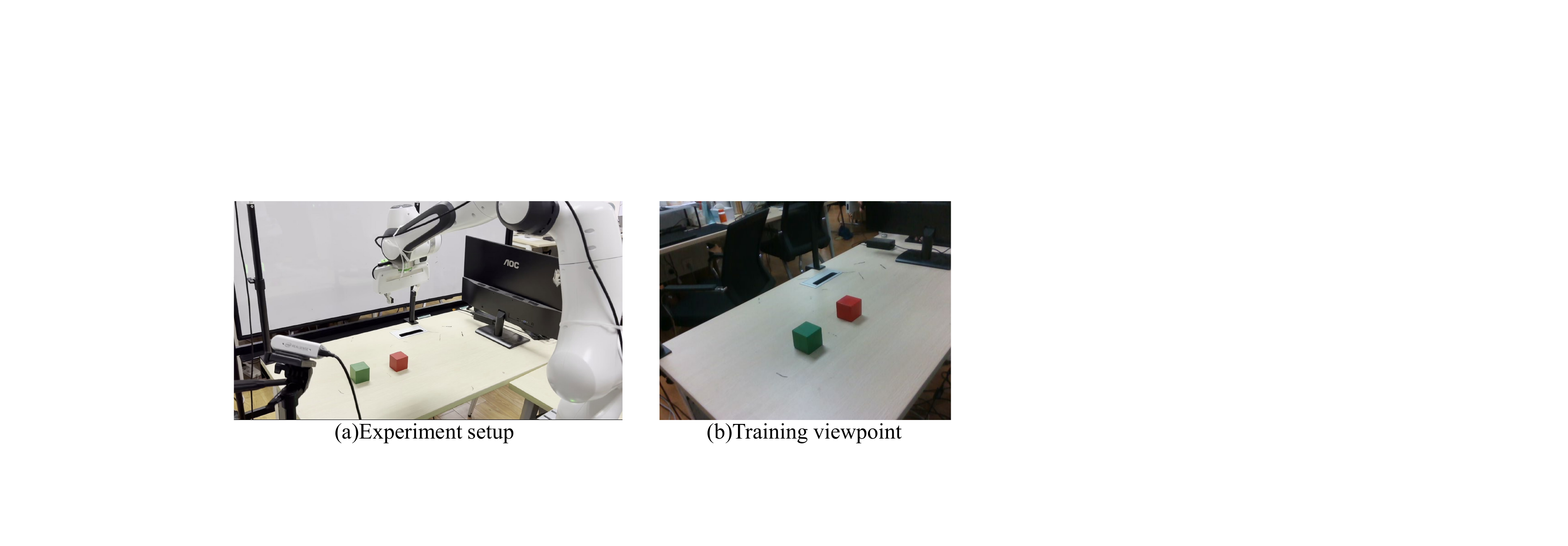}
    \caption{\textbf{Real-World Experimental Setup.} 
    (a) Our hardware environment features a Franka Emika Panda robot teleoperated via the GELLO framework, equipped with both a third-person static camera and a wrist-mounted camera. 
    (b) The \textit{Novel Camera Viewpoint} used for one-shot adaptation. This viewpoint introduces a significant spatial shift compared to the standard pre-training distribution, serving as the testbed for our Feature Linear Adaptation (FLA) method.
}
    \label{fig:real_setup}
\end{figure}

\begin{figure*}
    \centering
    \includegraphics[width=\linewidth]{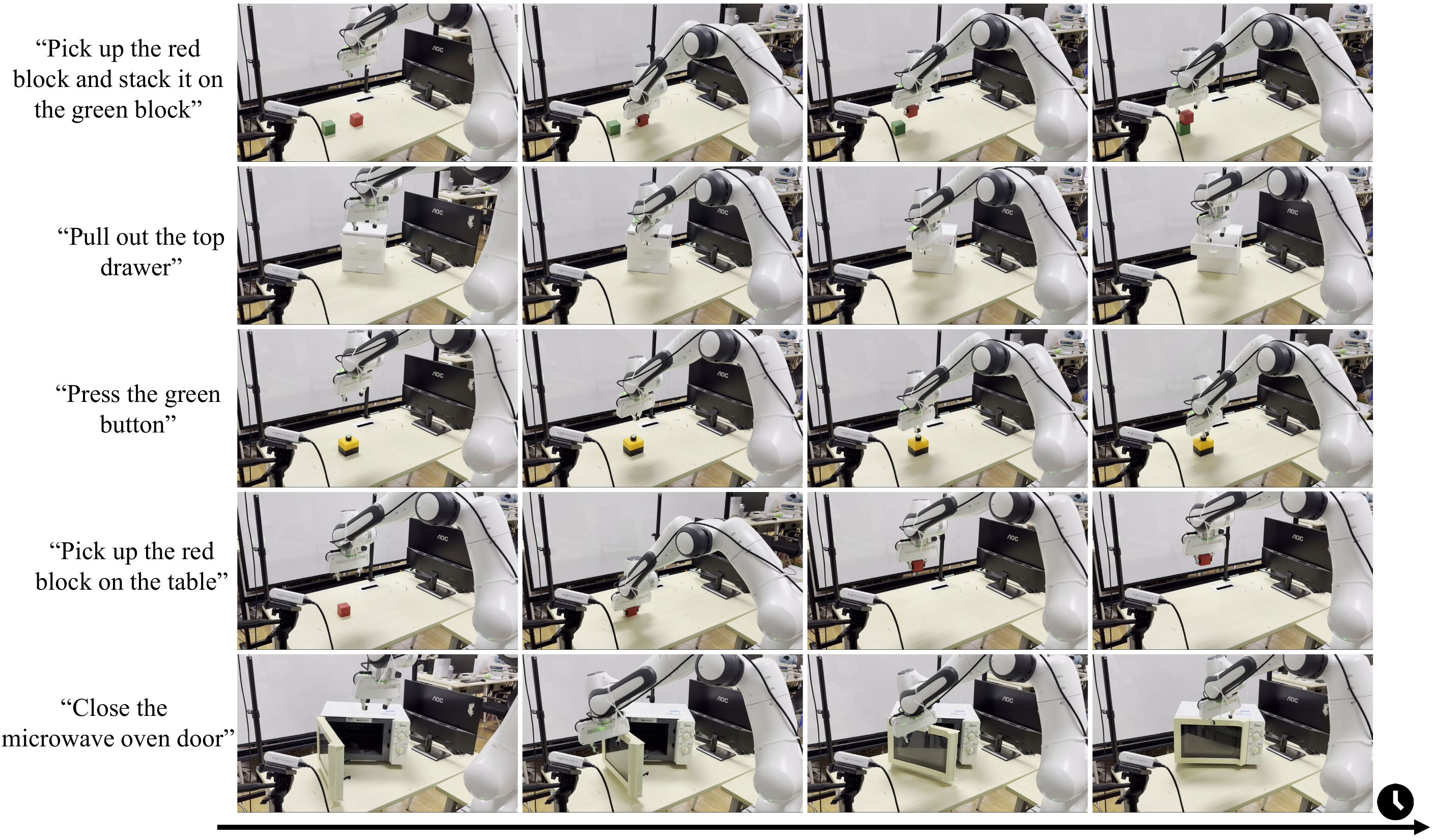}
    \caption{\textbf{Qualitative Results on Real-World Tasks.} 
    Sample rollouts of the five evaluation tasks executed by our policy after one-shot FLA adaptation. The tasks include: 
    (1) Pick up the red block, 
    (2) Stack the red block on the green block, 
    (3) Close the microwave oven door, 
    (4) Press the green button, and 
    (5) Pull out the top drawer. 
    Despite the substantial visual discrepancy introduced by the novel viewpoint, the adapted policy successfully recovers spatial grounding and executes precise manipulation in a closed-loop manner.
}
    \label{fig:real_tasks}
\end{figure*}

\subsection{Adaptation Protocol}
We evaluate the robustness of the pre-trained $\pi_{0.5}$ policy when transferred from a source viewpoint to a significantly different novel viewpoint (used in training, as shown in Fig.~\ref{fig:real_setup}).

Consistent with our simulation results, we employ \textbf{Feature Linear Adaptation (FLA)} for domain adaptation. Specifically, we apply FLA with a rank of $r=32$ to the linear layers of the Vision Transformer (ViT) encoder. The adaptation is performed in a \textbf{one-shot} manner: the model is updated using only a single human demonstration collected from the new viewpoint via GELLO. This setup tests whether our lightweight adaptation method can effectively realign the spatial representations of the physical world with minimal data and parameter overhead. We set the batch size to 32 and train for 750 steps.

Following the one-shot adaptation, we deploy the policy in a closed-loop control setting utilizing continuous dual-camera observations. We observe strong robustness to spatial misalignment and consistent success across all five tasks (see Fig.~\ref{fig:real_tasks}). These results empirically demonstrate that lightweight spatial adaptation is sufficient to bridge the domain gap between pre-trained distributions and novel real-world settings, effectively recovering the model's manipulation capabilities despite substantial visual discrepancies.

\subsection{Robustness to Imperfect Demonstrations}

During real-world data collection, human teleoperation can often yield imperfect or sub-optimal trajectories. As shown in Fig.~\ref{fig:imperfect_examples}(Adaptation Demonstrations), the one-shot demonstrations used for adaptation occasionally include imperfect executions, such as awkward grasps or slightly delayed alignments. Despite these imperfections in the source data, the Feature Linear Adaptation (FLA) method remains robust. For instance, FLA achieves a 90.7\% success rate on the $\pi_{0.5}$ baseline even when adapted using these sub-optimal trajectories. This indicates that our lightweight spatial adaptation primarily realigns the visual embedding space and is highly robust to minor kinematic imperfections in the demonstration data.

\subsection{Robustness to Dynamic Objects}

While our primary evaluations focus on static visual perturbations (e.g., camera viewpoints, lighting), real-world deployments often encounter dynamic environments. As illustrated in Fig.~\ref{fig:dynamic_objects}, we simulate a dynamic scenario by manually disturbing the target object's position (the blue block) during the robot's execution phase. Because our adapted policy operates in a closed-loop manner, it successfully perceives the displacement and adjusts its trajectory on the fly to grasp the moving object. This demonstrates that our adaptation framework is highly effective and robust not only against static visual shifts but also in dynamic environments.

\begin{figure}[htbp]
    \centering
    \includegraphics[width=\columnwidth]{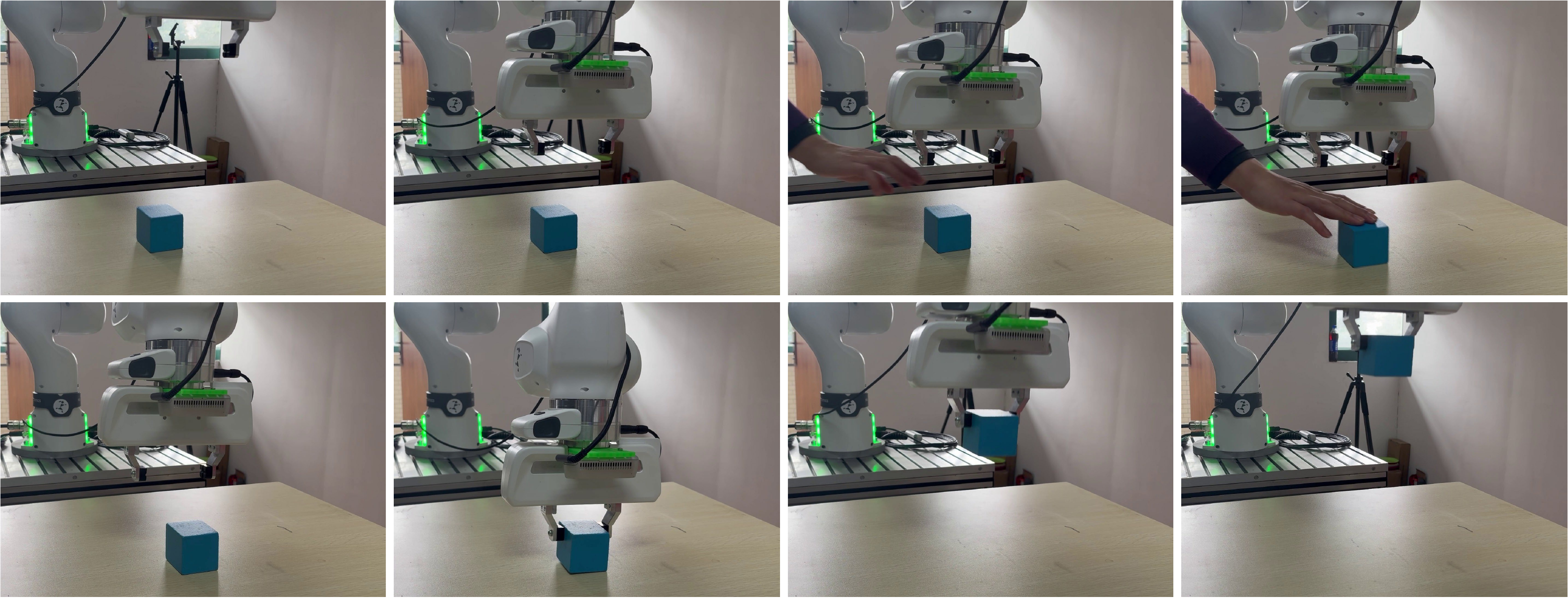} 
    \caption{\textbf{Dynamic objects.} FLA successfully adjusts trajectories in real-time against manual target disturbances.}
    \label{fig:dynamic_objects}
\end{figure}

\begin{figure}[htbp]
    \centering
    \includegraphics[width=\columnwidth]{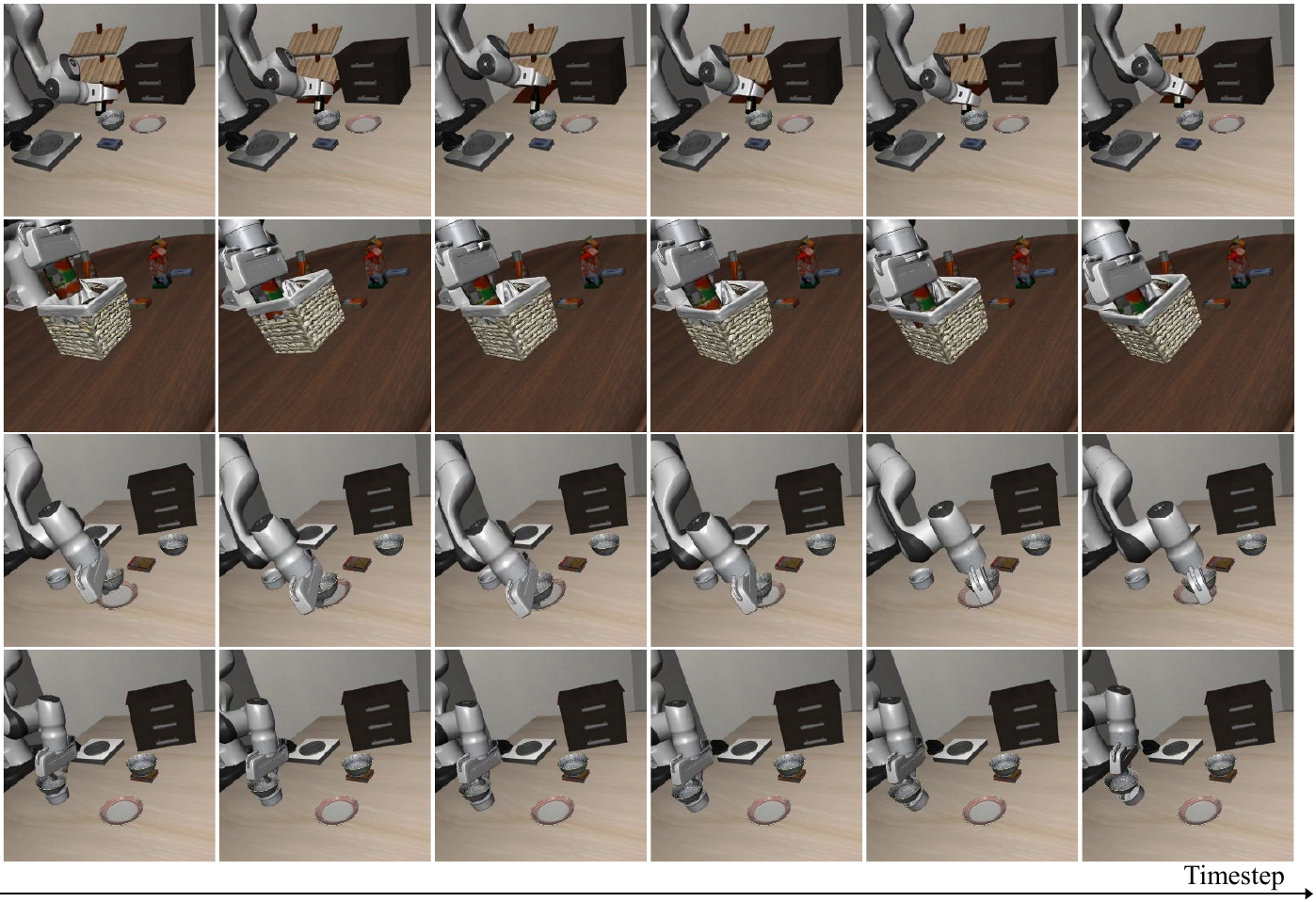} 
    \caption{\textbf{Imperfect examples.} FLA maintains high success rates even with sub-optimal demonstrations.}
    \label{fig:imperfect_examples}
\end{figure}

\section{Theoretical Analysis}

In this section, we provide theoretical justification for why lightweight
vision–level adapters such as the Feature Token Modulation (FTM) and the
Feature Linear Adaptation (FLA) are capable of restoring policy performance
under large visual distribution shifts (e.g., camera viewpoint changes).
We formalize the relationship between visual representation drift and
policy degradation, and show that small affine or low–rank corrections to
the vision encoder are sufficient to recover the original policy behavior
under mild and empirically verifiable assumptions.

\subsection{Preliminaries}

A vision–language–action (VLA) model is decomposed into a vision encoder
$f_v$, a language encoder $f_\ell$, and an autoregressive policy decoder
$g$:
\[
\hat a_t \sim g(a_{<t}; [z;\ell]), \qquad 
z = f_v(v),\quad \ell = f_\ell(l),
\]
where $z$ denotes the visual tokens passed to the decoder.

Let $\D_s$ denote the source (training) visual distribution and
$\D_t$ the target distribution after camera/viewpoint shifts.
Let
\[
Z_s = f_v(\D_s), \qquad Z_t = f_v(\D_t)
\]
be the corresponding distributions of visual tokens.

Throughout this section, we assume the action decoder's conditioning
context $(a_{<t}, \ell)$ is fixed, and we write $g_{a|z}$ for the induced
action distribution conditioned on visual token $z$.

We use $\dTV(\cdot,\cdot)$ to denote total variation distance.

\subsection{Assumptions}

We state the assumptions required for our theoretical results.

\paragraph{A1 (Locally Lipschitz Policy).}
There exists $L>0$ such that for all $z,z'$ in a neighborhood of
interest,
\[
\dTV\!\left(g_{a|z},\, g_{a|z'}\right)
\le L \| z - z' \|.
\]

\paragraph{A2 (Task Semantic Invariance).}
The underlying optimal action distribution does not change between
$\D_s$ and $\D_t$; i.e., degradation in policy performance arises solely
from drift in visual representations.

\paragraph{A3 (Locally Affine or Low–Rank Structure).}
The mapping from $Z_t$ to $Z_s$ can be well-approximated by an affine or
low–rank linear transformation in the neighborhood of interest.  
This corresponds to the FTM(Sec.~\ref{Feature Token Modulation}) and FLA(Sec.~\ref{Features Linear Adaptation}) parameterizations:
\[
\text{FTM:}\qquad \hat z = (1+\gamma)\odot z + \beta,
\]
\[
\text{FLA:}\qquad W' = W + \Delta W,\qquad \Delta W = BA,
\]
where $B\in\R^{d\times r}$ and $A\in\R^{r\times d}$ with $r\ll d$.

These assumptions are mild and can be empirically verified.

\subsection{Main Results}

We now formalize three theorems that connect representation drift with
policy degradation, and show why affine and low–rank corrections are
sufficient for recovery.

\begin{theorem}[Policy degradation is upper bounded by representation drift]
\label{thm:drift}
Let 
\[
z_s^\star = \E_{v\sim\D_s}\big[f_v(v)\big]
\]
be a representative source-domain token.
Under Assumption~A1,
\[
\E_{v\sim\D_t}\!
\left[
\dTV\!\left( g_{a|f_v(v)},\, g_{a|z_s^\star}\right)
\right]
\le
L\,
\E_{v\sim\D_t}
\left[
\| f_v(v) - z_s^\star \|
\right].
\]
\end{theorem}

\begin{proof}
By A1, for any $z,z'$,
\[
\dTV(g_{a|z}, g_{a|z'})
\le L \|z - z'\|.
\]
Taking $z = f_v(v)$ and $z' = z_s^\star$, and then taking expectation
over $v\sim\D_t$ yields the desired inequality.
\end{proof}

\paragraph{Interpretation.}
Performance degradation under viewpoint perturbations is controlled by
the drift of visual representations. If $Z_t$ deviates significantly from
$Z_s$, the policy shifts accordingly.

\begin{theorem}[Existence of affine corrections (justification of FTM)]
\label{thm:affine}
Suppose A1 and A3 hold, and that there exists an affine map
$A(z)=Mz+b$ (e.g., $M=\operatorname{diag}(1+\gamma)$ and $b=\beta$ as in
FTM) such that
\[
\E_{v\sim\D_t}
\| M f_v(v) + b - z_s^\star \|^2
\le \eps^2.
\]
Then the post-adaptation policy satisfies
\[
\E_{v\sim\D_t}
\dTV\!\left(
g_{a|A(f_v(v))},\,
g_{a|z_s^\star}
\right)
\le L\eps.
\]
\end{theorem}

\begin{proof}
Apply A1 with $z'=z_s^\star$:
\[
\dTV\!\left( g_{a|A(z)},\, g_{a|z_s^\star} \right)
\le L\,\|A(z)-z_s^\star\|.
\]
Taking $z=f_v(v)$ and expectation over $v\sim\D_t$ gives
\begin{align*}
\E\dTV &\le L\,\E\|A(f_v(v))-z_s^\star\| \\
       &\le L\sqrt{\E\|A(f_v(v))-z_s^\star\|^2} \\
       &\le L\eps
\end{align*}

where we used Jensen’s inequality.
\end{proof}

\paragraph{Interpretation.}
If target-domain tokens lie within an affine transformation of
source-domain tokens, then FTM---which implements a per-channel affine
transform—can provably restore the original policy up to error
$\mathcal{O}(\eps)$.

\begin{theorem}[Low-rank corrections approximate optimal linear shift (justification of FLA)]
\label{thm:lora}
Let $\Delta W^\star$ denote the optimal linear correction to a vision
layer (in any fixed Frobenius-norm sense).  
Let its singular values be $\sigma_1\ge\cdots\ge\sigma_d$.
Let $\Delta W_r$ be the best rank-$r$ approximation of $\Delta W^\star$.
Then by the Eckart–Young theorem,
\[
\|\Delta W^\star - \Delta W_r\|_F^2
=
\sum_{i=r+1}^{d} \sigma_i^2.
\]
If the representation drift is corrected using $\Delta W_r$ via
\[
W' = W + \Delta W_r = W + BA,
\]
then under A1,
\[
\dTV\!\left(
g_{a|z'},\, g_{a|z_s^\star}
\right)
\le
L\,\|z'-z_s^\star\|,
\]
and the residual representation error is controlled by the tail energy
$\sum_{i>r}\sigma_i^2$.
\end{theorem}

\begin{proof}
The Eckart–Young theorem states that $\Delta W_r$ is the optimal
rank-$r$ approximation of $\Delta W^\star$ under the Frobenius norm and
yields exactly the stated error.
This error propagates linearly through the affected network layer,
producing a representation error proportional to
$\|\Delta W^\star-\Delta W_r\|_F$.
Applying A1 then bounds the resulting deviation in policy outputs.
\end{proof}

\paragraph{Interpretation.}
If the optimal visual correction $\Delta W^\star$ is (approximately)
low-rank, then FLA---which imposes a rank-$r$ structure---recovers the
correction with error proportional to the spectrum tail.  
This explains why FLA achieves performance comparable to full LoRA with
orders of magnitude fewer parameters.

\subsection{Combined Error Bound}

Combining Theorems~\ref{thm:drift}--\ref{thm:lora}, for any adaptation
module $A_\phi$ (affine or low-rank), we obtain:
\[
\E_{v\sim\D_t}
\dTV\!\left(
g_{a|A_\phi(f_v(v))},\,
g_{a|z_s^\star}
\right)
\le
L\,
\E\|A_\phi(f_v(v)) - z_s^\star\|.
\]
Thus, the goal of adaptation is to minimize the right-hand side.  
Under A3, small-parameter modules such as FTM or FLA are expressive
enough to achieve this, thereby restoring the source-domain policy.

\section{Qualitative Analysis: Visual Embedding Alignment}
\label{sec:tsne_analysis}

To empirically validate our hypothesis that performance degradation under novel viewpoints stems from spatial misalignment, we visualized the distribution of visual tokens before and after adaptation. We extracted visual feature tokens from the SigLIP encoder using trajectories collected from two distinct domains: the source domain (original training camera) and the target domain (novel camera viewpoint). We employed t-Distributed Stochastic Neighbor Embedding (t-SNE) to project these high-dimensional embeddings into a 2D space, as shown in Figure~\ref{fig:compare_tsne}.

In the Zero-Shot setting (Figure~\ref{fig:compare_tsne}(a)), we observe a severe \textbf{embedding drift}. The visual tokens from the novel viewpoint (red) form a cluster that is completely isolated from the source domain manifold (blue). This significant domain gap explains the catastrophic failure of the pretrained policy, as the frozen physical modeling module receives inputs that lie outside its valid operating region.

In contrast, after applying our one-shot \textbf{Feature Linear Adaptation (FLA)} (Figure~\ref{fig:compare_tsne}(b)), the target distribution (green) undergoes a structured transformation. While preserving its internal geometric consistency, the target manifold is projected to closely adjoin the source manifold. This \textbf{manifold alignment} effectively bridges the domain gap, creating a continuous latent space that enables the frozen action expert to generalize to the new viewpoint without requiring full parameter finetuning.

\begin{figure*}
    \centering
    \includegraphics[width=\linewidth]{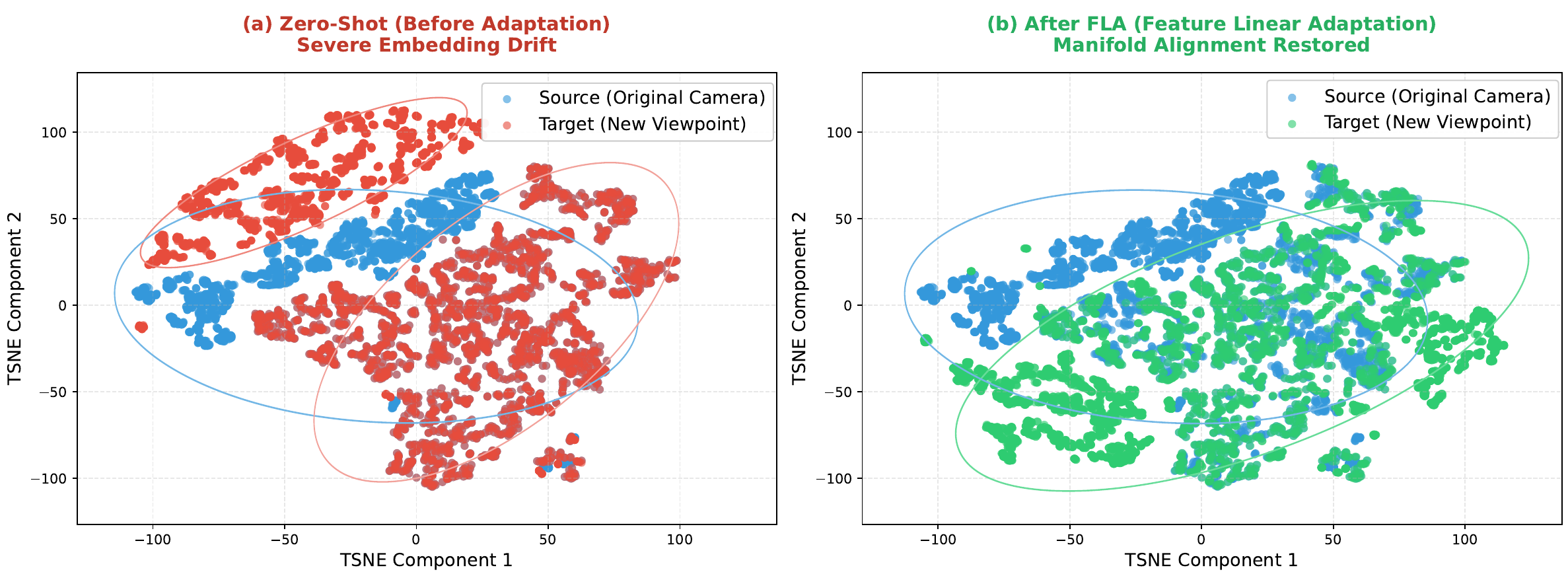}
    \caption{\textbf{Visualization of Visual Token Embeddings via t-SNE.} (a) \textbf{Before Adaptation:} A significant domain gap is observed between the source (blue) and target (red) embeddings, indicating severe spatial misalignment caused by viewpoint shifts. (b) \textbf{After FLA:} Our method projects the target embeddings (green) to align with the source manifold. This \textit{manifold alignment} restores the connectivity of the feature space, allowing the frozen policy to function correctly without requiring the distributions to perfectly overlap.}
    \label{fig:compare_tsne}
\end{figure*}

\section{Benchmark Details}
The original LIBERO benchmark \cite{liu2023liberobenchmarkingknowledgetransfer} assesses knowledge transfer in multitask and lifelong robot learning for manipulation tasks. It consists of four task suites, and each suite contains 10 tasks.

LIBERO-Spatial, which tests generalization to novel spatial layouts while keeping task and object types fixed;

LIBERO-Goal, which introduces new tasks while maintaining the same objects and layouts;

LIBERO-Object, which evaluates performance with novel object types under the same tasks and layouts; and

LIBERO-Long, which presents a broader variety of objects, layouts, and backgrounds.
\subsection{Robustness Extensions}
To rigorously evaluate the robustness of Vision-Language-Action (VLA) models against New Spatial Domains shifts beyond the original scope, we extended the standard LIBERO benchmark by implementing a multi-axis perturbation engine. Our implementation draws upon the evaluation protocols defined in Adapt3R \cite{wilcox2025adapt3radaptive3dscene} for geometric shifts and LIBERO-Plus \cite{fei2025liberoplusindepthrobustnessanalysis} for visual domain randomization. The benchmark systematically introduces variations across four distinct dimensions: Camera Pose, Illumination, Background Texture, and Sensor Noise.

\subsubsection{Camera Pose Variations}
Following the methodology in Adapt3R~\cite{wilcox2025adapt3radaptive3dscene}, we assess geometric robustness by altering the extrinsic parameters of the primary observation camera (referred to as \textit{agentview}). We implemented a custom \texttt{ControlEnv} wrapper that manipulates the camera's spatial position and orientation quaternions within the MuJoCo physics state during environment initialization. We support two variation modes:

\noindent\textbf{Continuous Orbital Rotation.} To test viewpoint invariance, we generate camera poses that orbit the robot's workspace. Let $\mathbf{p}_{cam}$ be the initial camera position and $\mathbf{p}_{eef}$ be the position of the robot's end-effector. We define a rotation radius $r = \|\mathbf{p}_{cam}^{xy} - \mathbf{p}_{eef}^{xy}\|_2$. Given a perturbation angle $\theta$, the new position $\mathbf{p}'_{cam}$ is computed as:
\[
    \mathbf{p}'_{cam} = \begin{bmatrix}
    r \cos(\theta) + \mathbf{p}_{eef}^x \\
    r \sin(\theta) + \mathbf{p}_{eef}^y \\
    \mathbf{p}_{cam}^z
    \end{bmatrix}
\]
The orientation is updated by composing the base rotation $R_{ref}$ with a Z-axis rotation $R_z(\theta)$ to ensure the camera remains focused on the workspace center.

\noindent\textbf{Discrete Perturbations.} For standardized New Spatial Domains evaluation, we defined three discrete difficulty levels (\textit{Small}, \textit{Medium}, \textit{Large}). These apply fixed translation offsets $\Delta \mathbf{p}$ and specific quaternion rotations to simulate progressively severe sensor misalignments.

Fig.~\ref{fig:score_bar_camera} shows the success rate of different viewpoints gained by discrete translation. Our method can efficiently adapt to camera variations across scales (e.g., large).

\subsubsection{Illumination Perturbations}
To evaluate robustness against illumination variations, we leverage the procedural scene generation pipeline from LIBERO-Plus~\cite{fei2025liberoplusindepthrobustnessanalysis}. In contrast to standard 2D image-space augmentations, we procedurally modify the internal MuJoCo scene configurations to introduce physically grounded lighting changes.
\begin{itemize}
    \item \textbf{Implementation:} We constructed a comprehensive library of scene definitions featuring diverse illumination profiles. To facilitate systematic evaluation, we implemented specialized environment classes that seamlessly load these pre-generated configurations.
    \item \textbf{Variation Factors:} We manipulate four attributes: (1) \textbf{Diffuse Color} (RGB intensity), (2) \textbf{Light Direction} (vector $\mathbf{d} \in \mathbb{R}^3$), (3) \textbf{Specular Intensity} (surface highlights), and (4) \textbf{Cast Shadows} (toggling shadow rendering for occlusions).
\end{itemize}

\subsubsection{Background Texture Perturbations}
We implemented a high-fidelity texture substitution system to test robustness against semantic visual shifts. Unlike simple color jittering, we utilize Physically Based Rendering (PBR) assets.
\begin{itemize}
    \item \textbf{Asset Library:} We utilized a curated dataset of high-resolution textures (up to 4K/6K) categorized into Natural (wood, stone), Industrial (metal, ceramic), and Fabric materials.
    \item \textbf{Dynamic Injection:} Upon environment initialization, we dynamically rebind the texture assets associated with the environmental boundaries (i.e., the floor and walls) within the simulation tree. This mechanism preserves the physical geometry of the scene while introducing significant variation in background visual styles.
\end{itemize}

\subsubsection{Sensor Noise Injection}
To simulate real-world camera imperfections, we implemented a real-time image degradation pipeline within the environment's \texttt{step()} function. This process intercepts the rendered RGB observation before it reaches the policy. We include five noise types with 10 severity levels each:
\begin{enumerate}
    \item \textbf{Motion Blur:} Applies directional convolution kernels to emulate the visual artifacts resulting from high-speed camera dynamics or robot movement.
    \item \textbf{Gaussian Blur:} Applies a Gaussian filter ($\sigma \in [1, 10]$) to simulate defocus.
    \item \textbf{Zoom Blur:} Averages scaled image crops to simulate focal length changes.
    \item \textbf{Fog:} Blends plasma fractal noise to simulate atmospheric scattering and contrast loss.
    \item \textbf{Glass Blur:} Combines local pixel shuffling with Gaussian smoothing to simulate textured occlusions.
\end{enumerate}

Table \ref{tab:benchmark_summary} summarizes the implementation scope for each perturbation dimension.

\begin{table*}[h]
    \centering
    \caption{Summary of the Robustness Benchmark Implementation.}
    \label{tab:benchmark_summary}
    \begin{tabular}{lll}
    \toprule
    \textbf{Dimension} & \textbf{Implementation Method} & \textbf{Key Parameters} \\
    \midrule
    Camera Pose & Direct MuJoCo State Manipulation & Orbit $\theta$, Offset $\Delta \mathbf{p}$, Quaternion $q$ \\
    Lighting & Procedural XML Generation & Diffuse, Specular, Direction, Shadows \\
    Texture & Dynamic Asset Swapping & PBR Materials (Wood, Stone, Metal) \\
    Sensor Noise & Real-time Post-processing & Blur (Motion/Gaussian/Zoom), Fog, Noise \\
    \bottomrule
    \end{tabular}
\end{table*}

\begin{figure}
    \centering
    \includegraphics[width=0.94\linewidth]{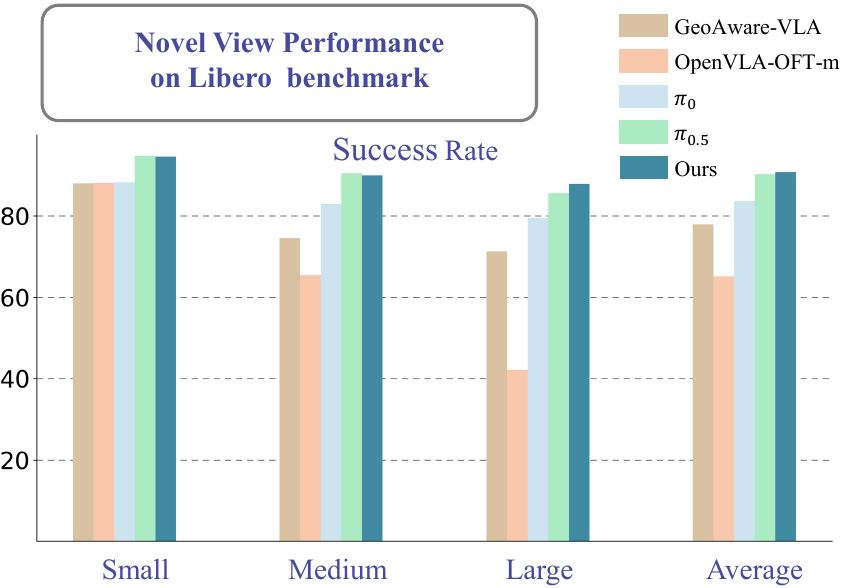}
    \caption{Success rates under Small, Medium, and Large camera viewpoint shifts, showing the stability and robustness of our adapted policy.}
    \label{fig:score_bar_camera}
\end{figure}

\section{Model Details}
\label{sec:model_details}

We provide a rigorous description of the $\pi_{0.5}$ formulation, detailing the probabilistic decomposition, the derivation of the flow matching objective, and the specific architectural mechanisms used for timestep conditioning.

\subsection{Probabilistic Formulation and Architecture}
The $\pi_{0.5}$ model is designed to capture a joint distribution over high-level semantic outputs $\hat{l}$ (e.g., subtask descriptions, VQA answers) and low-level continuous action trajectories $a_{t:t+H}$, conditioned on multimodal observations $o_t$ and task instructions $l$.
$\pi_{0.5}$ decomposes this joint distribution into a hierarchical inference process:
\[
\pi_{\theta}(a_{t:t+H}, \hat{l} | o_t, l) = 
\underbrace{\pi_{\theta}(a_{t:t+H} | o_t, \hat{l})}_{\substack{\text{Action Expert} \\ \text{(Flow)}}} 
\cdot 
\underbrace{\pi_{\theta}(\hat{l} | o_t, l)}_{\substack{\text{VLM Backbone} \\ \text{(Discrete)}}}
\]
This decomposition motivates our dual-stream architecture:
\begin{itemize}
    \item \textbf{High-Level Policy ($\pi_{\theta}(\hat{l} | \cdot)$):} Parameterized by the PaliGemma 2B VLM backbone, optimized via cross-entropy to reason about subtasks and scene semantics.
    \item \textbf{Low-Level Policy ($\pi_{\theta}(a | \cdot)$):} Parameterized by the Action Expert, optimized via flow matching to generate precise control signals conditioned on the backbone's embeddings.
\end{itemize}

\subsection{VLM Backbone Implementation}
The backbone processes a unified sequence of tokens $x_{1:N}$ including images, text, and discrete proprioception.
\begin{itemize}
    \item \textbf{Architecture:} $\pi_{0.5}$ utilize the PaliGemma architecture (Width=2048, Depth=18, Heads=18, MLP=16384).
    \item \textbf{Bidirectional Attention:} Unlike standard causal masking $M_{ij} = \mathbb{I}(j \le i)$, $\pi_{0.5}$ employs bidirectional attention for the prefix sequence (images $I_{t}$ and prompt $l$) to ensure full context visibility: $A(x_i, x_j) = 1$ for all $x_i, x_j \in \text{Prefix}$.
    \item \textbf{Discrete Objective:} The discrete loss $\mathcal{L}_{\text{discrete}}$ is the standard negative log-likelihood over the target tokens:
    \[
        \mathcal{L}_{\text{discrete}} = - \sum_{m=1}^{M} \log P_{\theta}(x_m | x_{1:m-1}, o_t, l)
    \]
    This objective jointly supervises semantic text generation and discrete action prediction (via FAST tokens).
\end{itemize}

\subsection{Flow Matching Objective}
To model the continuous action distribution $\pi_{\theta}(a_{t:t+H} | \cdot)$, $\pi_{0.5}$ employ Conditional Flow Matching (CFM).

\textbf{Interpolation and Vector Field:}
$\pi_{0.5}$ defines a probability path that transforms a Gaussian noise distribution $p_0(\omega) = \mathcal{N}(\omega; 0, I)$ to the data distribution $p_1(a)$ (robot actions).
$\pi_{0.5}$ constructs intermediate samples $a^{\tau}$ at timestep $\tau \in [0, 1]$ using linear interpolation between noise $\omega$ and data $a_{t:t+H}$:
\[
    a^{\tau} = \psi_t(\omega, a_{t:t+H}) = \tau a_{t:t+H} + (1-\tau)\omega
\]
The objective is to learn a time-dependent vector field $v_t(a^{\tau})$ that generates this flow. By differentiating the interpolation path with respect to $\tau$, $\pi_{0.5}$ derives the target vector field:
\[
    u_t(a^{\tau} | a_{t:t+H}, \omega) = \frac{d}{d\tau} (\tau a_{t:t+H} + (1-\tau)\omega) = a_{t:t+H} - \omega
\]
\textit{Note: Following the implementation in $\pi_{0.5}$, the network targets the direction $\omega - a_{t:t+H}$ (flow from data to noise) or vice versa. The loss function minimizes the error between the predicted field $f_{\theta}^{a}$ and this target.}

\textbf{Flow Matching Loss:}
The Action Expert is trained to regress this vector field:
\[
\begin{aligned}
    \mathcal{L}_{\text{flow}} 
    &= \mathbb{E}_{\tau \sim p(\tau), \omega \sim \mathcal{N}(0,I), a \sim \mathcal{D}} \Big[ || (\omega - a_{t:t+H}) \\
    &\quad 
    - f_{\theta}^{a}(a^{\tau}, o_t, l, \tau) ||^2 \big]
\end{aligned}
\]

where the timestep $\tau$ is sampled from a Beta distribution $\text{Beta}(1.5, 1)$ to emphasize learning at low-noise levels ($s=0.999$).

\textbf{Inference (Integration):}
At inference time, $\pi_{0.5}$ approximates the integral of the predicted vector field using an ODE solver (Euler method) with $K=10$ steps:
\[
    a^{\tau + \Delta \tau} \leftarrow a^{\tau} + \Delta \tau \cdot f_{\theta}^{a}(a^{\tau}, o_t, \hat{l}, \tau)
\]

\subsection{Action Expert Conditioning Mechanism}
The Action Expert (300M parameters) is a Transformer distinct from the backbone. To strictly condition the continuous generation on the flow timestep $\tau$, $\pi_{0.5}$ utilizes Adaptive RMSNorm.

\textbf{Timestep Injection:}
The scalar timestep $\tau$ is first embedded via sinusoidal positional encodings $\phi(\tau)$ and an MLP $W$. This embedding modulates the normalization layer of each transformer block:
\[
\begin{split}
    \text{AdaRMSNorm}(x, \tau) &= y \cdot (1 + \gamma(\tau)) + \beta(\tau), \\
    \text{where } y &= \frac{x}{\|x\|_2}
\end{split}
\]
Here, $\gamma(\tau)$ and $\beta(\tau)$ are learned scale and shift parameters projected from the timestep embedding. This ensures that temporal information is injected globally into the expert's feature space.

\subsection{Attention Masking Strategy}
To enable the joint optimization objective $\mathcal{L} = \mathcal{L}_{\text{discrete}} + \alpha \mathcal{L}_{\text{flow}}$, $\pi_{0.5}$ employs a structured attention mask $M \in \{0, 1\}^{N \times N}$:
\begin{enumerate}
    \item \textbf{VLM $\rightarrow$ Expert:} The Action Expert attends to the VLM's visual and textual embeddings ($M_{expert, vlm} = 1$).
    \item \textbf{Expert $\nrightarrow$ VLM:} The VLM backbone is prevented from attending to the Action Expert ($M_{vlm, expert} = 0$) to preserve the pre-trained semantic representations.
    \item \textbf{Information Barrier:} Crucially, the Action Expert is masked from the discrete FAST action tokens ($M_{expert, FAST} = 0$). This conditional independence assertion forces the expert to infer $a_{t:t+H}$ solely from $o_t$ and $l$, preventing information leakage from the discrete solution.
\end{enumerate}

\section{Training Details for baselines}

In this section, we detail the training configurations for all baselines and our proposed methods (FTM and FLA). The evaluated baselines cover a broad spectrum of adaptation strategies, including parameter-efficient fine-tuning (LoRA), prompt learning, and visual backbone replacement. To ensure fair comparisons, all experiments utilize the same pre-trained checkpoints and evaluation protocols. 

Before detailing specific hyperparameters, we analyze the training dynamics. As shown in Fig.~\ref{fig:solution}, we compare the one-shot adaptation performance of our Feature Linear Adaptation (FLA) against LoRA and full fine-tuning across $\pi_0$ and $\pi_{0.5}$ architectures. FLA achieves faster convergence and greater stability across training steps without overfitting. Notably, FLA reaches comparable or slightly superior peak success rates relative to LoRA, and consistently outperforms full fine-tuning, while updating only a fraction of the parameters.

We now outline the specific optimization and preprocessing configurations, beginning with the $\pi_{0.5}$ LoRA baseline.

\vspace{1em}
\noindent \textbf{$\pi_{0.5}$ LoRA Baseline}

We fine-tune the $\pi_{0.5}$ backbone equipped with the pi05 action head using our custom dataset as shown in Fig. ~\ref{fig:solution}(a).

\begin{enumerate}[wide=0pt, label=\textbf{\arabic*.}, labelsep=0.5em, itemsep=0.3em]
    \item \textbf{Model Configuration:} To ensure parameter efficiency, we adopt Low-Rank Adaptation (LoRA) while keeping all non-LoRA weights frozen via the model-provided freeze filter. Specifically, we configure the VLM component using the Gemma 2B LoRA variant with a LoRA rank of 16, and the action expert using the Gemma 300M LoRA variant with a LoRA rank of 32, applying adaptation to both attention and feed-forward network (FFN) layers.
    
    \item \textbf{Optimization:} The optimization process strictly adheres to the default VLA pretraining protocols. We utilize the AdamW optimizer with momentum parameters $\beta_1=0.9$, $\beta_2=0.95$, $\epsilon=10^{-8}$, a weight decay of $1 \times 10^{-10}$, and a global gradient norm clipping threshold of 1.0. The learning rate follows a cosine warmup schedule (1,000-step warmup), reaching a peak of $2.5 \times 10^{-5}$ and decaying to a minimum of $2.5 \times 10^{-6}$. 
    
    \item \textbf{Data Preprocessing:} Visual observations are resized to $224 \times 224$ using a padding strategy (Resize with Pad).
\end{enumerate}

\begin{table*}[t]
    \centering
    \caption{Hyperparameters for our proposed adaptation methods (Feature Token Modulation and Feature Linear Adaptation).}
    \label{tab:our_solutions_params}
    \resizebox{0.75\linewidth}{!}{
    \begin{tabular}{l|cc}
    \toprule
    \textbf{Configuration} & \textbf{Feature Token Modulation (FTM)} & \textbf{Feature Linear Adaptation (FLA)} \\
    \midrule
    Optimizer & AdamW, clip norm 1.0 & AdamW, clip norm 1.0 \\
    Optimizer Momentum & $\beta_1=0.9, \beta_2=0.95, \epsilon=1e-8$ & $\beta_1=0.9, \beta_2=0.95, \epsilon=1e-8$ \\
    Weight Decay & $1 \times 10^{-10}$ & $1 \times 10^{-10}$ \\
    Batch Size & 32 & 32 \\
    Training Iterations & 5,000 steps & 1,500 steps \\
    Cosine Warmup Steps & 500 & 500 \\
    Peak Learning Rate & $5 \times 10^{-4}$ & $5 \times 10^{-4}$ \\
    Cosine Decay Steps & 5,000 & 2,000 \\
    Min LR (Decay Target) & $5 \times 10^{-5}$ & $5 \times 10^{-6}$ \\
    Image Resize & $224 \times 224$ & $224 \times 224$ \\
    \bottomrule
    \end{tabular}
    }
\end{table*}

\begin{figure}
    \centering
    \includegraphics[width=\linewidth]{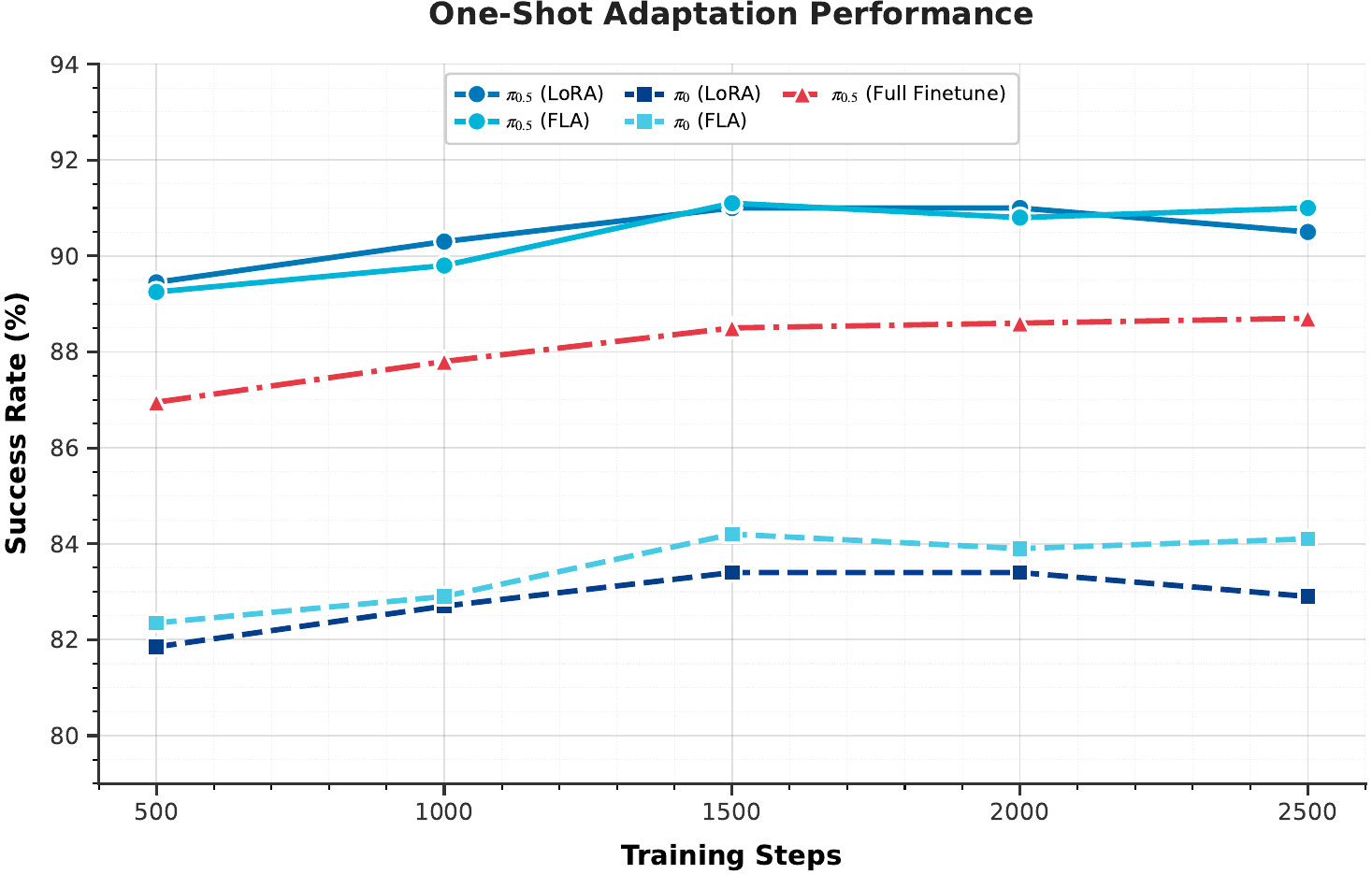}
    \caption{\textbf{Adaptation Stability.} Success rates across training steps. FLA remains stable and consistently outperforms baselines without overfitting.}
    \label{fig:train_dot}
\end{figure}

\vspace{1em}
\noindent \textbf{Alternative Visual Backbones Baseline}

To evaluate the efficacy of replacing the visual encoder with a robust geometric backbone, we implement a baseline following the GeoAware-VLA architecture as shown in Fig. ~\ref{fig:solution}(b).

\begin{enumerate}[wide=0pt, label=\textbf{\arabic*.}, labelsep=0.5em, itemsep=0.3em]
    \item \textbf{Architecture:} The standard visual encoder is replaced by a pre-trained Visual Geometry Grounded Transformer (VGGT), which serves as a frozen feature extractor. To align features, we introduce a trainable Feature Projection Layer that processes features from four evenly spaced intermediate layers via a 1D convolutional network, adaptive average pooling, and a final MLP projector to produce the visual embedding $z_{vis}$.
    
    \item \textbf{Optimization:} The policy (excluding the frozen vision backbone) is trained using the AdamW optimizer with standard momentum parameters ($\beta_1=0.9, \beta_2=0.999$) and a weight decay of $1 \times 10^{-4}$. We employ a constant learning rate of $1 \times 10^{-4}$ and a global batch size of 64.
    
    \item \textbf{Data Preprocessing:} Visual observations are rendered at a resolution of $128 \times 128$ and subsequently center-cropped to $126 \times 126$ pixels to match the input requirements.
\end{enumerate}

\vspace{1em}
\noindent \textbf{Prompt-based Adaptation Baseline}

As illustrated in Fig. ~\ref{fig:solution}(c),  we adopt a standard prompt tuning strategy by introducing learnable embedding tokens into the input sequence to implement the prompt learning baseline.

\begin{enumerate}[wide=0pt, label=\textbf{\arabic*.}, labelsep=0.5em, itemsep=0.3em]
    \item \textbf{Mechanism:} We instantiate learnable prompt tokens (initialized from a normal distribution with $\sigma=0.02$) that are directly concatenated to the model inputs. These are divided into two groups: Prefix Prompts, which match the embedding width of the frozen PaliGemma encoder and are concatenated to the multimodal (image and language) sequence; and Suffix Prompts, which align with the action expert's width and are inserted between the proprioceptive state embedding and the action tokens. Unlike modulation-based approaches (e.g., FiLM), no auxiliary projection networks or affine transformations are applied. These tokens function purely as additional learnable context within the Transformer's attention mechanism. Consistent with the prompt tuning paradigm, the pre-trained backbone remains entirely frozen; optimization is restricted exclusively to the learnable prompt tokens and the action head.
    
    \item \textbf{Optimization:} We utilize the AdamW optimizer ($\beta_1=0.9, \beta_2=0.95, \epsilon=10^{-8}$, weight decay $1 \times 10^{-10}$) with a global gradient norm clipping of 1.0. Distinct from the LoRA baseline, the learning rate follows a schedule tuned for adaptation: a cosine schedule with a 500-step warmup, reaching a peak of $5 \times 10^{-4}$ and decaying to $5 \times 10^{-5}$ over a total of 2,000 iterations (batch size 32).
    
    \item \textbf{Data Preprocessing:} Training is conducted on our custom dataset. Visual inputs undergo the same preprocessing pipeline as the LoRA baseline: frames are resized to $224 \times 224$ with padding.
\end{enumerate}

\vspace{1em}
\noindent \textbf{Training Details for Our Solutions}

We introduce two complementary, parameter-efficient adaptation strategies: \textbf{Feature Token Modulation (FTM)} and \textbf{Feature Linear Adaptation (FLA)}. Both methods focus on adapting pre-trained representations while keeping the majority of the VLA backbone frozen. Detailed hyperparameters are provided in Table~\ref{tab:our_solutions_params}. As illustrated in Fig.~\ref{fig:train_dot}, FLA demonstrates strong training stability. The success rate consistently improves and plateaus across training steps without suffering from catastrophic overfitting, confirming its robustness to hyperparameter choices.

\begin{enumerate}[wide=0pt, label=\textbf{\arabic*.}, labelsep=0.5em, itemsep=0.3em]
    \item \textbf{Feature Token Modulation (FTM):} As illustrated in Fig.~\ref{fig:solution}(d), FTM applies a global affine transformation to the visual embedding space. We instantiate two learnable tokens, global scale ($\gamma$) and shift ($\beta$) parameters. The VLA backbone is entirely frozen; we optimize \textit{only} the learnable prompts. Optimization utilizes the AdamW optimizer ($\beta_1=0.9, \beta_2=0.95, \epsilon=10^{-8}$, weight decay $1 \times 10^{-10}$) with a global gradient norm clipping of 1.0. The learning rate follows a cosine schedule with a 500-step warmup, reaching a peak of $5 \times 10^{-4}$ and decaying to $5 \times 10^{-5}$ over 5,000 iterations. The batch size is set to 32.
    
    \item \textbf{Feature Linear Adaptation (FLA):} As shown in Fig. ~\ref{fig:solution}(e), FLA targets internal feature adaptation by injecting LoRA adapters ($\Delta W = BA$) into the linear layers (27 layers) of the SigLIP ViT encoder. We freeze the language model and action expert, updating only the LoRA weights of linear layer in ViT. Optimization settings generally mirror the FTM configuration (AdamW with consistent $\beta$ and $\epsilon$ values), but the learning rate schedule is adjusted to decay to a lower minimum of $5 \times 10^{-6}$, and EMA is disabled. Training is conducted with float32 precision for 1,500 iterations using a global batch size of 32.
    
    \item \textbf{Common Preprocessing:} For both strategies, visual inputs undergo the standard preprocessing pipeline: images are resized to $224 \times 224$.
\end{enumerate}

\begin{table}[h] 
    \centering
    \caption{Success Rate (\%) of FLA across different VLA models under novel camera viewpoints.}
    \label{tab:extended_rank_ablation}
    \resizebox{0.95\linewidth}{!}{
    \begin{tabular}{lccccc}
        \toprule
        \textbf{Model} & \textbf{Zero-shot} & \textbf{Rank=8} & \textbf{Rank=16} & \textbf{Rank=32} & \textbf{Rank=64} \\
        \midrule
        OpenVLA-OFT (One-Shot FLA) & 50.3 & 89.0  & 88.0  & 88.0  & 85.6 \\
        $\pi_0$ (One-Shot FLA)       & 49.5 & 82.8  & 84.0  & 84.5  & 84.8 \\
        $\pi_{0.5}$ (One-Shot FLA)   & 48.5 & 91.05 & 90.8  & 91.2  & 90.6 \\
        \bottomrule
    \end{tabular}
    }
\end{table}

\section{Additional Ablation Studies on Adaptation Rank}
\label{sec:additional_ablation}

In this section, we provide extended ablation results regarding the adaptation rank ($r$) used in our Feature Linear Adaptation (FLA) method. Table~\ref{tab:extended_rank_ablation} presents the one-shot adaptation success rates across three distinct base architectures (OpenVLA-OFT, $\pi_0$, and $\pi_{0.5}$) under novel camera viewpoints. The results indicate that while higher ranks generally offer increased capacity for feature realignment, even modest ranks (e.g., $r=16$) achieve highly competitive robustness with minimal parameter overhead.

%% file: main.bib
@String(IJCV = {Int. J. Comput. Vis.})

@String(CVPR= {IEEE Conf. Comput. Vis. Pattern Recog.})

@String(ICLR = {Int. Conf. Learn. Represent.})

@String(IJCV  = {IJCV})

@String(CVPR  = {CVPR})

@String(ICLR  = {ICLR})

@article{chen2026radar,
  title={RADAR: Benchmarking Vision-Language-Action Generalization via Real-World Dynamics, Spatial-Physical Intelligence, and Autonomous Evaluation},
  author={Chen, Yuhao and Zhan, Zhihao and Lin, Xiaoxin and Song, Zijian and Liu, Hao and Lyu, Qinhan and Zu, Yubo and Chen, Xiao and Liu, Zhiyuan and Pu, Tao and others},
  journal={arXiv preprint arXiv:2602.10980},
  year={2026}
}

@article{song2026learning,
  title={Learning Physics from Pretrained Video Models: A Multimodal Continuous and Sequential World Interaction Models for Robotic Manipulation},
  author={Song, Zijian and Li, Qichang and Qin, Sihan and Chen, Yuhao and Chen, Tianshui and Lin, Liang and Wang, Guangrun},
  journal={arXiv preprint arXiv:2603.00110},
  year={2026}
}

@misc{brohan2023rt2visionlanguageactionmodelstransfer,
      title={RT-2: Vision-Language-Action Models Transfer Web Knowledge to Robotic Control}, 
      author={Anthony Brohan and Noah Brown and Justice Carbajal and Yevgen Chebotar and Xi Chen and Krzysztof Choromanski and Tianli Ding and Danny Driess and others},
      year={2023},
      eprint={2307.15818},
      archivePrefix={arXiv},
      primaryClass={cs.RO},
      url={https://arxiv.org/abs/2307.15818}, 
}

@misc{embodimentcollaboration2025openxembodimentroboticlearning,
      title={Open X-Embodiment: Robotic Learning Datasets and RT-X Models}, 
      author={Embodiment Collaboration and Abby O'Neill and Abdul Rehman and Abhinav Gupta and Abhiram Maddukuri and Abhishek Gupta and Abhishek Padalkar and Abraham Lee and Acorn Pooley and Agrim Gupta and Ajay Mandlekar and Ajinkya Jain and Albert Tung and Alex Bewley and Alex Herzog and Alex Irpan and Alexander Khazatsky and others},
      year={2025},
      eprint={2310.08864},
      archivePrefix={arXiv},
      primaryClass={cs.RO},
      url={https://arxiv.org/abs/2310.08864}, 
}

@misc{li2024visionlanguagefoundationmodelseffective,
      title={Vision-Language Foundation Models as Effective Robot Imitators}, 
      author={Xinghang Li and Minghuan Liu and Hanbo Zhang and Cunjun Yu and Jie Xu and Hongtao Wu and Chilam Cheang and Ya Jing and Weinan Zhang and Huaping Liu and Hang Li and Tao Kong},
      year={2024},
      eprint={2311.01378},
      archivePrefix={arXiv},
      primaryClass={cs.RO},
      url={https://arxiv.org/abs/2311.01378}, 
}

@misc{kim2024openvlaopensourcevisionlanguageactionmodel,
      title={OpenVLA: An Open-Source Vision-Language-Action Model}, 
      author={Moo Jin Kim and Karl Pertsch and Siddharth Karamcheti and Ted Xiao and Ashwin Balakrishna and Suraj Nair and Rafael Rafailov and Ethan Foster and Grace Lam and Pannag Sanketi and Quan Vuong and Thomas Kollar and Benjamin Burchfiel and Russ Tedrake and Dorsa Sadigh and Sergey Levine and Percy Liang and Chelsea Finn},
      year={2024},
      eprint={2406.09246},
      archivePrefix={arXiv},
      primaryClass={cs.RO},
      url={https://arxiv.org/abs/2406.09246}, 
}

@misc{durante2024interactiveagentfoundationmodel,
      title={An Interactive Agent Foundation Model}, 
      author={Zane Durante and Bidipta Sarkar and Ran Gong and Rohan Taori and Yusuke Noda and Paul Tang and Ehsan Adeli and Shrinidhi Kowshika Lakshmikanth and Kevin Schulman and Arnold Milstein and Demetri Terzopoulos and Ade Famoti and Noboru Kuno and Ashley Llorens and Hoi Vo and Katsu Ikeuchi and Li Fei-Fei and Jianfeng Gao and Naoki Wake and Qiuyuan Huang},
      year={2024},
      eprint={2402.05929},
      archivePrefix={arXiv},
      primaryClass={cs.AI},
      url={https://arxiv.org/abs/2402.05929}, 
}

@misc{huang2024embodiedgeneralistagent3d,
      title={An Embodied Generalist Agent in 3D World}, 
      author={Jiangyong Huang and Silong Yong and Xiaojian Ma and Xiongkun Linghu and Puhao Li and Yan Wang and Qing Li and Song-Chun Zhu and Baoxiong Jia and Siyuan Huang},
      year={2024},
      eprint={2311.12871},
      archivePrefix={arXiv},
      primaryClass={cs.CV},
      url={https://arxiv.org/abs/2311.12871}, 
}

@misc{zhen20243dvla3dvisionlanguageactiongenerative,
      title={3D-VLA: A 3D Vision-Language-Action Generative World Model}, 
      author={Haoyu Zhen and Xiaowen Qiu and Peihao Chen and Jincheng Yang and Xin Yan and Yilun Du and Yining Hong and Chuang Gan},
      year={2024},
      eprint={2403.09631},
      archivePrefix={arXiv},
      primaryClass={cs.CV},
      url={https://arxiv.org/abs/2403.09631}, 
}

@misc{black2024pi0visionlanguageactionflowmodel,
      title={$\pi_0$: A Vision-Language-Action Flow Model for General Robot Control}, 
      author={Kevin Black and Noah Brown and Danny Driess and Adnan Esmail and Michael Equi and Chelsea Finn and Niccolo Fusai and Lachy Groom and Karol Hausman and Brian Ichter and Szymon Jakubczak and Tim Jones and Liyiming Ke and Sergey Levine and Adrian Li-Bell and Mohith Mothukuri and Suraj Nair and Karl Pertsch and Lucy Xiaoyang Shi and James Tanner and Quan Vuong and Anna Walling and Haohuan Wang and Ury Zhilinsky},
      year={2024},
      eprint={2410.24164},
      archivePrefix={arXiv},
      primaryClass={cs.LG},
      url={https://arxiv.org/abs/2410.24164}, 
}

@misc{khazatsky2025droidlargescaleinthewildrobot,
      title={DROID: A Large-Scale In-The-Wild Robot Manipulation Dataset}, 
      author={Alexander Khazatsky and Karl Pertsch and Suraj Nair and Ashwin Balakrishna and Sudeep Dasari and Siddharth Karamcheti and Soroush Nasiriany and Mohan Kumar Srirama and others},
      year={2025},
      eprint={2403.12945},
      archivePrefix={arXiv},
      primaryClass={cs.RO},
      url={https://arxiv.org/abs/2403.12945}, 
}

@misc{walke2024bridgedatav2datasetrobot,
      title={BridgeData V2: A Dataset for Robot Learning at Scale}, 
      author={Homer Walke and Kevin Black and Abraham Lee and Moo Jin Kim and Max Du and Chongyi Zheng and Tony Zhao and Philippe Hansen-Estruch and Quan Vuong and Andre He and Vivek Myers and Kuan Fang and Chelsea Finn and Sergey Levine},
      year={2024},
      eprint={2308.12952},
      archivePrefix={arXiv},
      primaryClass={cs.RO},
      url={https://arxiv.org/abs/2308.12952}, 
}

@misc{octomodelteam2024octoopensourcegeneralistrobot,
      title={Octo: An Open-Source Generalist Robot Policy}, 
      author={Octo Model Team and Dibya Ghosh and Homer Walke and Karl Pertsch and Kevin Black and Oier Mees and Sudeep Dasari and Joey Hejna and Tobias Kreiman and Charles Xu and Jianlan Luo and You Liang Tan and Lawrence Yunliang Chen and Pannag Sanketi and Quan Vuong and Ted Xiao and Dorsa Sadigh and Chelsea Finn and Sergey Levine},
      year={2024},
      eprint={2405.12213},
      archivePrefix={arXiv},
      primaryClass={cs.RO},
      url={https://arxiv.org/abs/2405.12213}, 
}

@misc{li2024evaluatingrealworldrobotmanipulation,
      title={Evaluating Real-World Robot Manipulation Policies in Simulation}, 
      author={Xuanlin Li and Kyle Hsu and Jiayuan Gu and Karl Pertsch and Oier Mees and Homer Rich Walke and Chuyuan Fu and Ishikaa Lunawat and Isabel Sieh and Sean Kirmani and Sergey Levine and Jiajun Wu and Chelsea Finn and Hao Su and Quan Vuong and Ted Xiao},
      year={2024},
      eprint={2405.05941},
      archivePrefix={arXiv},
      primaryClass={cs.RO},
      url={https://arxiv.org/abs/2405.05941}, 
}

@misc{xie2023decomposinggeneralizationgapimitation,
      title={Decomposing the Generalization Gap in Imitation Learning for Visual Robotic Manipulation}, 
      author={Annie Xie and Lisa Lee and Ted Xiao and Chelsea Finn},
      year={2023},
      eprint={2307.03659},
      archivePrefix={arXiv},
      primaryClass={cs.RO},
      url={https://arxiv.org/abs/2307.03659}, 
}

@misc{zhou2025liberoprorobustfairevaluation,
      title={LIBERO-PRO: Towards Robust and Fair Evaluation of Vision-Language-Action Models Beyond Memorization}, 
      author={Xueyang Zhou and Yangming Xu and Guiyao Tie and Yongchao Chen and Guowen Zhang and Duanfeng Chu and Pan Zhou and Lichao Sun},
      year={2025},
      eprint={2510.03827},
      archivePrefix={arXiv},
      primaryClass={cs.CV},
      url={https://arxiv.org/abs/2510.03827}, 
}

@misc{fei2025liberoplusindepthrobustnessanalysis,
      title={LIBERO-Plus: In-depth Robustness Analysis of Vision-Language-Action Models}, 
      author={Senyu Fei and Siyin Wang and Junhao Shi and Zihao Dai and Jikun Cai and Pengfang Qian and Li Ji and Xinzhe He and Shiduo Zhang and Zhaoye Fei and Jinlan Fu and Jingjing Gong and Xipeng Qiu},
      year={2025},
      eprint={2510.13626},
      archivePrefix={arXiv},
      primaryClass={cs.RO},
      url={https://arxiv.org/abs/2510.13626}, 
}

@misc{li2024llavaonevisioneasyvisualtask,
      title={LLaVA-OneVision: Easy Visual Task Transfer}, 
      author={Bo Li and Yuanhan Zhang and Dong Guo and Renrui Zhang and Feng Li and Hao Zhang and Kaichen Zhang and Peiyuan Zhang and Yanwei Li and Ziwei Liu and Chunyuan Li},
      year={2024},
      eprint={2408.03326},
      archivePrefix={arXiv},
      primaryClass={cs.CV},
      url={https://arxiv.org/abs/2408.03326}, 
}

@inproceedings{xing2021kitchenshift,
    title={KitchenShift: Evaluating Zero-Shot Generalization of Imitation-Based Policy Learning Under Domain Shifts},
    author={Xing, Eliot and Gupta, Abhinav and Powers*, Sam and Dean*, Victoria},
    booktitle={NeurIPS 2021 Workshop on Distribution Shifts: Connecting Methods and Applications},
    year={2021},
    url={https://openreview.net/forum?id=DdglKo8hBq0}
}

@misc{yu2021metaworldbenchmarkevaluationmultitask,
      title={Meta-World: A Benchmark and Evaluation for Multi-Task and Meta Reinforcement Learning}, 
      author={Tianhe Yu and Deirdre Quillen and Zhanpeng He and Ryan Julian and Avnish Narayan and Hayden Shively and Adithya Bellathur and Karol Hausman and Chelsea Finn and Sergey Levine},
      year={2021},
      eprint={1910.10897},
      archivePrefix={arXiv},
      primaryClass={cs.LG},
      url={https://arxiv.org/abs/1910.10897}, 
}

@misc{zhang2024vlabenchlargescalebenchmarklanguageconditioned,
      title={VLABench: A Large-Scale Benchmark for Language-Conditioned Robotics Manipulation with Long-Horizon Reasoning Tasks}, 
      author={Shiduo Zhang and Zhe Xu and Peiju Liu and Xiaopeng Yu and Yuan Li and Qinghui Gao and Zhaoye Fei and Zhangyue Yin and Zuxuan Wu and Yu-Gang Jiang and Xipeng Qiu},
      year={2024},
      eprint={2412.18194},
      archivePrefix={arXiv},
      primaryClass={cs.RO},
      url={https://arxiv.org/abs/2412.18194}, 
}

@misc{liu2023liberobenchmarkingknowledgetransfer,
      title={LIBERO: Benchmarking Knowledge Transfer for Lifelong Robot Learning}, 
      author={Bo Liu and Yifeng Zhu and Chongkai Gao and Yihao Feng and Qiang Liu and Yuke Zhu and Peter Stone},
      year={2023},
      eprint={2306.03310},
      archivePrefix={arXiv},
      primaryClass={cs.AI},
      url={https://arxiv.org/abs/2306.03310}, 
}

@misc{wilcox2025adapt3radaptive3dscene,
      title={Adapt3R: Adaptive 3D Scene Representation for Domain Transfer in Imitation Learning}, 
      author={Albert Wilcox and Mohamed Ghanem and Masoud Moghani and Pierre Barroso and Benjamin Joffe and Animesh Garg},
      year={2025},
      eprint={2503.04877},
      archivePrefix={arXiv},
      primaryClass={cs.CV},
      url={https://arxiv.org/abs/2503.04877}, 
}

@misc{yang2025fp33dfoundationpolicy,
      title={FP3: A 3D Foundation Policy for Robotic Manipulation}, 
      author={Rujia Yang and Geng Chen and Chuan Wen and Yang Gao},
      year={2025},
      eprint={2503.08950},
      archivePrefix={arXiv},
      primaryClass={cs.RO},
      url={https://arxiv.org/abs/2503.08950}, 
}

@misc{ze20243ddiffusionpolicygeneralizable,
      title={3D Diffusion Policy: Generalizable Visuomotor Policy Learning via Simple 3D Representations}, 
      author={Yanjie Ze and Gu Zhang and Kangning Zhang and Chenyuan Hu and Muhan Wang and Huazhe Xu},
      year={2024},
      eprint={2403.03954},
      archivePrefix={arXiv},
      primaryClass={cs.RO},
      url={https://arxiv.org/abs/2403.03954}, 
}

@misc{abouzeid2025geoawarevlaimplicitgeometryaware,
      title={GeoAware-VLA: Implicit Geometry Aware Vision-Language-Action Model}, 
      author={Ali Abouzeid and Malak Mansour and Zezhou Sun and Dezhen Song},
      year={2025},
      eprint={2509.14117},
      archivePrefix={arXiv},
      primaryClass={cs.RO},
      url={https://arxiv.org/abs/2509.14117}, 
}

@inproceedings{pang2025reviwo,
  title={Learning View-invariant World Models for Visual Robotic Manipulation},
  author={Jing-Cheng Pang and 
	  Nan Tang and 
	  kaiyuan Li and 
	  Yuting Tang and 
	  Xin-Qiang Cai and 
	  Zhen-Yu Zhang and 
	  Gang Niu and 
	  Sugiyama Masashi and 
	  Yang Yu},
  booktitle={International Conference on Learning Representations (ICLR)},
  year={2025}
}

@misc{kim2025finetuningvisionlanguageactionmodelsoptimizing,
      title={Fine-Tuning Vision-Language-Action Models: Optimizing Speed and Success}, 
      author={Moo Jin Kim and Chelsea Finn and Percy Liang},
      year={2025},
      eprint={2502.19645},
      archivePrefix={arXiv},
      primaryClass={cs.RO},
      url={https://arxiv.org/abs/2502.19645}, 
}

@misc{finn2017modelagnosticmetalearningfastadaptation,
      title={Model-Agnostic Meta-Learning for Fast Adaptation of Deep Networks}, 
      author={Chelsea Finn and Pieter Abbeel and Sergey Levine},
      year={2017},
      eprint={1703.03400},
      archivePrefix={arXiv},
      primaryClass={cs.LG},
      url={https://arxiv.org/abs/1703.03400}, 
}

@inproceedings{
hu2022lora,
title={Lo{RA}: Low-Rank Adaptation of Large Language Models},
author={Edward J Hu and yelong shen and Phillip Wallis and Zeyuan Allen-Zhu and Yuanzhi Li and Shean Wang and Lu Wang and Weizhu Chen},
booktitle={International Conference on Learning Representations},
year={2022},
url={https://openreview.net/forum?id=nZeVKeeFYf9}
}

@inproceedings{driess2023palme,
    title={PaLM-E: An Embodied Multimodal Language Model},
    author={Driess, Danny and Xia, Fei and Sajjadi, Mehdi S. M. and Lynch, Corey and Chowdhery, Aakanksha and Ichter, Brian and Wahid, Ayzaan and Tompson, Jonathan and Vuong, Quan and Yu, Tianhe and Huang, Wenlong and Chebotar, Yevgen and Sermanet, Pierre and Duckworth, Daniel and Levine, Sergey and Vanhoucke, Vincent and Hausman, Karol and Toussaint, Marc and Greff, Klaus and Zeng, Andy and Mordatch, Igor and Florence, Pete},
    booktitle={arXiv preprint arXiv:2303.03378},
    year={2023}
}

@misc{openai2024gpt4technicalreport,
      title={GPT-4 Technical Report}, 
      author={OpenAI and Josh Achiam and Steven Adler and Sandhini Agarwal and Lama Ahmad and Ilge Akkaya and Florencia Leoni Aleman and Diogo Almeida and Janko Altenschmidt and Sam Altman and Shyamal Anadkat and Red Avila and others},
      year={2024},
      eprint={2303.08774},
      archivePrefix={arXiv},
      primaryClass={cs.CL},
      url={https://arxiv.org/abs/2303.08774}, 
}

@misc{geminiteam2025geminifamilyhighlycapable,
      title={Gemini: A Family of Highly Capable Multimodal Models}, 
      author={Gemini Team and Rohan Anil and Sebastian Borgeaud and Jean-Baptiste Alayrac and Jiahui Yu and Radu Soricut and Johan Schalkwyk and Andrew M. Dai and Anja Hauth and Katie Millican and others},
      year={2025},
      eprint={2312.11805},
      archivePrefix={arXiv},
      primaryClass={cs.CL},
      url={https://arxiv.org/abs/2312.11805}, 
}

@inproceedings{zhou2022cocoop,
    title={Conditional Prompt Learning for Vision-Language Models},
    author={Zhou, Kaiyang and Yang, Jingkang and Loy, Chen Change and Liu, Ziwei},
    booktitle={IEEE/CVF Conference on Computer Vision and Pattern Recognition (CVPR)},
    year={2022}
}

@article{zhou2022coop,
    title={Learning to Prompt for Vision-Language Models},
    author={Zhou, Kaiyang and Yang, Jingkang and Loy, Chen Change and Liu, Ziwei},
    journal={International Journal of Computer Vision (IJCV)},
    year={2022}
}

@misc{houlsby2019parameterefficienttransferlearningnlp,
      title={Parameter-Efficient Transfer Learning for NLP}, 
      author={Neil Houlsby and Andrei Giurgiu and Stanislaw Jastrzebski and Bruna Morrone and Quentin de Laroussilhe and Andrea Gesmundo and Mona Attariyan and Sylvain Gelly},
      year={2019},
      eprint={1902.00751},
      archivePrefix={arXiv},
      primaryClass={cs.LG},
      url={https://arxiv.org/abs/1902.00751}, 
}

@misc{li2021prefixtuningoptimizingcontinuousprompts,
      title={Prefix-Tuning: Optimizing Continuous Prompts for Generation}, 
      author={Xiang Lisa Li and Percy Liang},
      year={2021},
      eprint={2101.00190},
      archivePrefix={arXiv},
      primaryClass={cs.CL},
      url={https://arxiv.org/abs/2101.00190}, 
}

@misc{zaken2022bitfitsimpleparameterefficientfinetuning,
      title={BitFit: Simple Parameter-efficient Fine-tuning for Transformer-based Masked Language-models}, 
      author={Elad Ben Zaken and Shauli Ravfogel and Yoav Goldberg},
      year={2022},
      eprint={2106.10199},
      archivePrefix={arXiv},
      primaryClass={cs.LG},
      url={https://arxiv.org/abs/2106.10199}, 
}

@misc{intelligence2025pi05visionlanguageactionmodelopenworld,
      title={$\pi_{0.5}$: a Vision-Language-Action Model with Open-World Generalization}, 
      author={Physical Intelligence and Kevin Black and Noah Brown and James Darpinian and Karan Dhabalia and Danny Driess and Adnan Esmail and Michael Equi and Chelsea Finn and Niccolo Fusai and Manuel Y. Galliker and Dibya Ghosh and Lachy Groom and Karol Hausman and Brian Ichter and Szymon Jakubczak and Tim Jones and Liyiming Ke and Devin LeBlanc and Sergey Levine and Adrian Li-Bell and Mohith Mothukuri and Suraj Nair and Karl Pertsch and Allen Z. Ren and Lucy Xiaoyang Shi and Laura Smith and Jost Tobias Springenberg and Kyle Stachowicz and James Tanner and Quan Vuong and Homer Walke and Anna Walling and Haohuan Wang and Lili Yu and Ury Zhilinsky},
      year={2025},
      eprint={2504.16054},
      archivePrefix={arXiv},
      primaryClass={cs.LG},
      url={https://arxiv.org/abs/2504.16054}, 
}

@misc{tschannen2025siglip2multilingualvisionlanguage,
      title={SigLIP 2: Multilingual Vision-Language Encoders with Improved Semantic Understanding, Localization, and Dense Features}, 
      author={Michael Tschannen and Alexey Gritsenko and Xiao Wang and Muhammad Ferjad Naeem and Ibrahim Alabdulmohsin and Nikhil Parthasarathy and Talfan Evans and Lucas Beyer and Ye Xia and Basil Mustafa and Olivier Hénaff and Jeremiah Harmsen and Andreas Steiner and Xiaohua Zhai},
      year={2025},
      eprint={2502.14786},
      archivePrefix={arXiv},
      primaryClass={cs.CV},
      url={https://arxiv.org/abs/2502.14786}, 
}

@misc{wu2023gello,
    title={GELLO: A General, Low-Cost, and Intuitive Teleoperation Framework for Robot Manipulators},
    author={Philipp Wu and Yide Shentu and Zhongke Yi and Xingyu Lin and Pieter Abbeel},
    year={2023},
}

@misc{zhan2026stablelanguageguidancevisionlanguageaction,
      title={Stable Language Guidance for Vision-Language-Action Models}, 
      author={Zhihao Zhan and Yuhao Chen and Jiaying Zhou and Qinhan Lv and Hao Liu and Keze Wang and Liang Lin and Guangrun Wang},
      year={2026},
      eprint={2601.04052},
      archivePrefix={arXiv},
      primaryClass={cs.RO},
      url={https://arxiv.org/abs/2601.04052}, 
}

@misc{zhou2026tagtargetagnosticguidancestable,
      title={TAG: Target-Agnostic Guidance for Stable Object-Centric Inference in Vision-Language-Action Models}, 
      author={Jiaying Zhou and Zhihao Zhan and Ruifeng Zhai and Qinhan Lyu and Hao Liu and Keze Wang and Liang Lin and Guangrun Wang},
      year={2026},
      eprint={2603.24584},
      archivePrefix={arXiv},
      primaryClass={cs.CV},
      url={https://arxiv.org/abs/2603.24584}, 
}

@article{song2025physical,
  title={Physical autoregressive model for robotic manipulation without action pretraining},
  author={Song, Zijian and Qin, Sihan and Chen, Tianshui and Lin, Liang and Wang, Guangrun},
  journal={arXiv preprint arXiv:2508.09822},
  year={2025}
}

@misc{zhan2026e0enhancinggeneralizationfinegrained,
      title={E0: Enhancing Generalization and Fine-Grained Control in VLA Models via Tweedie Discrete Diffusion}, 
      author={Zhihao Zhan and Jiaying Zhou and Likui Zhang and Qinhan Lv and Hao Liu and Jusheng Zhang and Weizheng Li and Ziliang Chen and Tianshui Chen and Ruifeng Zhai and Keze Wang and Liang Lin and Guangrun Wang},
      year={2026},
      eprint={2511.21542},
      archivePrefix={arXiv},
      primaryClass={cs.RO},
      url={https://arxiv.org/abs/2511.21542}, 
}

@misc{qu2025spatialvlaexploringspatialrepresentations,
      title={SpatialVLA: Exploring Spatial Representations for Visual-Language-Action Model}, 
      author={Delin Qu and Haoming Song and Qizhi Chen and Yuanqi Yao and Xinyi Ye and Yan Ding and Zhigang Wang and JiaYuan Gu and Bin Zhao and Dong Wang and Xuelong Li},
      year={2025},
      eprint={2501.15830},
      archivePrefix={arXiv},
      primaryClass={cs.RO},
      url={https://arxiv.org/abs/2501.15830}, 
}

@misc{zhang2026vlm4vlarevisitingvisionlanguagemodelsvisionlanguageaction,
      title={VLM4VLA: Revisiting Vision-Language-Models in Vision-Language-Action Models}, 
      author={Jianke Zhang and Xiaoyu Chen and Qiuyue Wang and Mingsheng Li and Yanjiang Guo and Yucheng Hu and Jiajun Zhang and Shuai Bai and Junyang Lin and Jianyu Chen},
      year={2026},
      eprint={2601.03309},
      archivePrefix={arXiv},
      primaryClass={cs.CV},
      url={https://arxiv.org/abs/2601.03309}, 
}
